\documentclass{article} 
\usepackage{iclr2026_conference,times}
\iclrfinalcopy
\usepackage{graphicx}
\usepackage{xcolor}
\usepackage{booktabs}
\usepackage{amsmath,amssymb,amsthm}
\usepackage{subcaption}
\newtheorem{assumption}{Assumption}[section]

\newtheorem{lemma}{Lemma}[section]
\newtheorem{proposition}{Proposition}[section]
\newtheorem{corollary}{Corollary}[section]
\theoremstyle{remark}
\newtheorem{remark}{Remark}[section]

\definecolor{retainpurple}{RGB}{120, 70, 180}
\definecolor{forgetgreen}{RGB}{40, 150, 90}

\usepackage{amsmath,amssymb,amsfonts,bm}

\def\eqref#1{equation~\ref{#1}}

\def\1{\bm{1}}

\DeclareMathAlphabet{\mathsfit}{\encodingdefault}{\sfdefault}{m}{sl}
\SetMathAlphabet{\mathsfit}{bold}{\encodingdefault}{\sfdefault}{bx}{n}

\usepackage{hyperref}
\usepackage{url}

\title{Neuralyzing the Trace: Selective Representation-Level Unlearning with Contrastive Sparse Autoencoders}

\author{
Itai Zehavi\textsuperscript{1,2,3}
\quad
Fanny Jourdan\textsuperscript{4,1}
\quad
Ulrich Aïvodji\textsuperscript{2,1}
\\
\textsuperscript{1}Mila
\quad
\textsuperscript{2}ILLS
\quad
\textsuperscript{3}École normale supérieure Paris-Saclay
\quad
\textsuperscript{4}IRT Saint Exupéry
\\
\texttt{itai.zehavi@ens-paris-saclay.fr}
}

\begin{document}

\maketitle

\begin{abstract}
Machine unlearning aims to remove targeted information while preserving a
model's other abilities. In realistic settings, such as privacy requests
under the EU GDPR, the target may be narrow, for example information
associated with a single person. Behavioral forgetting alone may be
insufficient, motivating interventions directly on internal
representations. However, standard mechanistic-interpretability extractors are poorly
selective for such targets. We identify an \emph{energy bias} in
reconstruction-based extraction, which favors dominant background
structure over low-energy target-specific components. We introduce \textsc{SCALPEL}, a contrastive sparse autoencoder designed
to learn more selective forget features. We show theoretically that
contrastive training promotes target-selective features and that our
selection score controls expected background knowledge perturbation. We validate \textsc{SCALPEL} experimentally on TOFU across Qwen, Llama,
and Gemma, where it substantially improves over NMF and standard SAE
interventions and is competitive with Gradient Difference and RMU,
bridging mechanistic interpretability and fine-grained unlearning.
\end{abstract}

\section{Introduction}
\label{sec:intro}

\begin{figure}[t]
    \centering
    \includegraphics[width=\textwidth]{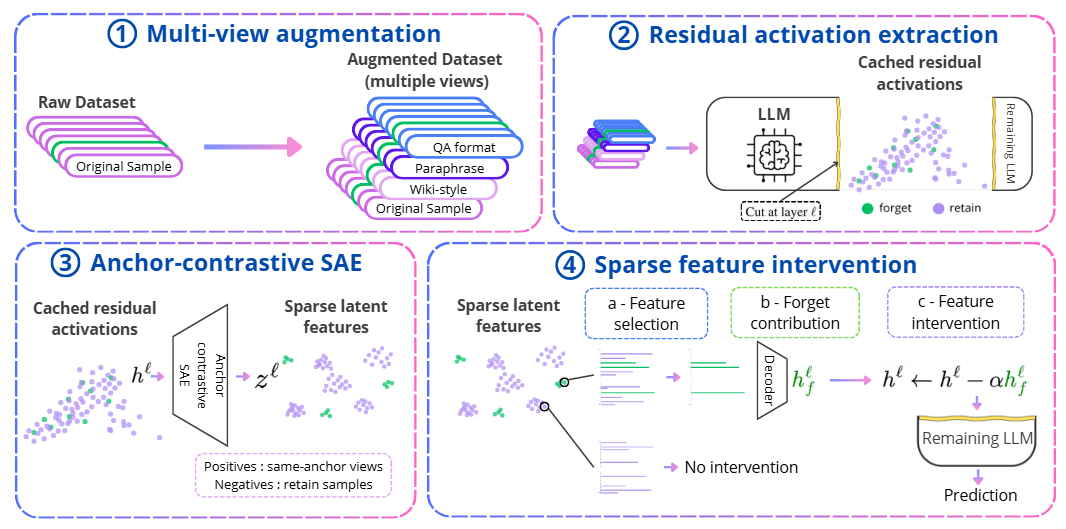}
    \caption{
    Overview of the \textsc{SCALPEL} pipeline.
    \circnum{1} We generate multiple semantic views of each forget anchor,
    then \circnum{2} cache residual-stream activations at layer $\ell$.
    \circnum{3} An anchor-contrastive SAE learns sparse features using
    same-anchor views as positives and \textcolor{retainpurple}{retain
    samples} (purple) as negatives, so that \textcolor{forgetgreen}{forget
    samples} (green) become isolated from retain information in the sparse
    feature space.
    \circnum{4} We select forget-specific features, decode their
    contribution \textcolor{forgetgreen}{$h_f^\ell$}, and subtract it from
    the residual stream before the remaining layers produce the
    prediction. The same selected feature set is applied to every input, while subtraction occurs only when those features activate. The intervention therefore acts as an input-dependent internal filter intended to suppress forget-related components while limiting its effect on retain activations.
    }
    \label{fig:method_overview}
\end{figure}

Once a large language model (LLM) has learned a piece of information, it is hard to take back.
Yet a private fact, a copyrighted passage, or a narrow sensitive association may
need to be removed without retraining from scratch. This is the goal of machine unlearning, and it matters directly for privacy, model governance, and legal requirements such as the right to
erasure (``right to be forgotten'') established in Article~17 of the EU
General Data Protection Regulation (GDPR)~\citep{eu2016gdpr}.

Most existing methods treat unlearning as an optimization problem on the
model weights: gradient-difference methods ~\citep{bu2024ngdiff}, retain-only finetuning,
representation steering such as RMU~\citep{huutien2024adaptiveRMU}, or in-context
unlearning~\citep{pawelczyk2023incontextunlearning}.
These methods can make the model stop answering forget-set queries. But
they leave a basic question open: did the model remove the
knowledge, or did it only learn not to express it?

If the internal representation survives,
the knowledge can resurface under paraphrases, indirect prompts,
jailbreak-style queries, or a small amount of relearning. Recent work on
parametric knowledge traces shows that some unlearning methods change
behavior while leaving much of the underlying trace
intact~\citep{hong2025parametrictraces}. This motivates a more direct approach: identify target-associated
representations inside the model and intervene on them explicitly.

Mechanistic interpretability provides natural tools for this. Transformer
computation can be read through the residual stream, where information accumulates
across layers~\citep{elhage2021framework}, and concept-extraction methods such as
PCA, NMF, and sparse autoencoders (SAEs) decompose activations into interpretable
directions~\citep{fel2023holistic,jourdan2023cockatiel,cunningham2023sparse,bricken2023monosemanticity,gao2024scalingSAE}.
But these tools work best when the target is a broad concept already well separated
in activation space~\citep{farrell2024saeunlearning,patil2024neurosurgeon}, whereas
realistic requests are far more precise: forgetting one author is not removing a
whole topic. In this regime we find that standard PCA- and SAE-based interventions
are poorly selective, and the features they remove damage retain knowledge almost as
much as forget knowledge.

One important cause lies in the
training objective. Indeed, Reconstruction-based extractors are rewarded in
proportion to the \emph{energy} (dispersion) of what they capture, and a specific
forget target is a low-energy (low dispersion) component hidden inside a much larger
background (high dispersion) (see Section~\ref{sec:energy-bias}). The extractor is therefore more strongly rewarded for modeling the background than for isolating the target.

We address this with \textsc{SCALPEL}, a contrastive sparse autoencoder
built on a simple hypothesis: specific knowledge is easier to isolate when
the extractor is trained to focus on what stays stable across many
views of the same target. We augment each forget anchor into paraphrases
and rewritten variants, treat views of the same anchor as positives, and
draw negatives from retain data rather than from other forget examples.
This avoids pushing related forget examples apart and pulls
target-specific structure away from shared language and task features.
Once trained, the SAE provides a fixed set of forget-associated sparse
features. The same set is used for every input, while each feature is
subtracted only when activated by the encoder. This produces an explicit
and input-dependent internal filter with tunable strength
(Figure~\ref{fig:method_overview}).

The goal of this paper is not to solve machine unlearning in full or to
provide certified deletion from the model parameters. Instead, we take a
first step toward connecting mechanistic interpretability with targeted
unlearning by studying whether internal features can be extracted with
enough selectivity to support narrow forget requests.

\paragraph{Contributions.}
We study this question in a controlled author-level setting on TOFU and
show that adapting the \emph{feature-learning objective} can produce
selective internal features that support competitive representation-level
unlearning while remaining explicit and inspectable. Our contributions are:
\begin{itemize}
    \item We identify and formalize a \textbf{selectivity bottleneck}:
    reconstruction-trained extractors are energy-weighted and may therefore
    underrepresent low-energy forget targets
    (Section~\ref{sec:energy-bias}).

    \item We introduce \textbf{\textsc{SCALPEL}}, an anchor-contrastive sparse
    autoencoder that learns features stable across augmented views of a target
    and weak on background data, together with a \textbf{norm-aware selection
    criterion} accounting for both encoder activation and decoder-direction
    magnitude (Sections~\ref{sec:scalpel}--\ref{sec:feature-selection}).

    \item We provide a \textbf{theoretical account} of the method:
    reconstruction objectives are energy-weighted, adding contrastive learning
    favors more target-selective sparse representations, and high-score
    features admit a bound on expected background perturbation
    (Section~\ref{sec:theory}).

    \item We \textbf{validate the approach empirically} on TOFU across three
    model families, showing improved feature selectivity and competitive
    unlearning performance against both mechanistic-interpretability and
    optimization-based baselines (Section~\ref{sec:results}).
\end{itemize}

\section{Related Work}
\label{sec:related-work}

\paragraph{General approaches to LLM unlearning.}
Unlearning removes the influence of a forget set $\mathcal{D}_f$ while
preserving behavior on a retain set $\mathcal{D}_r$; an auxiliary set
$\mathcal{D}_w$ (e.g.\ Wiki data) is often used to check general
knowledge~\citep{maini2024tofu,shi2024muse,dorna2025openunlearning}. A
first family updates the weights with gradient-based objectives. Gradient
Ascent raises the loss on forget examples but quickly damages the
model~\citep{maini2024tofu}; Gradient Difference adds descent on
$\mathcal{D}_r$ to limit the damage, with normalized
variants~\citep{maini2024tofu,bu2024ngdiff}; Negative Preference
Optimization replaces the unstable ascent term with a smoother
preference-based loss~\citep{zhang2024npo}; other approaches finetune the
model toward generic or incorrect answers~\citep{eldan2023harrypotter}.
Unlearning can also avoid weight changes: in-context unlearning conditions
the model at inference time~\citep{pawelczyk2023incontextunlearning},
which changes behavior but not the stored representation. Model editing
such as ROME updates parameters tied to one factual
association~\citep{meng2022rome}, though editing replaces an association
rather than removing an influence.

A particularly relevant baseline is Representation Misdirection for Unlearning
(RMU)~\citep{li2024wmdp,huutien2024adaptiveRMU}, which pushes forget activations at
a chosen layer toward a random direction while pinning retain activations to those
of the frozen original model. RMU shares our core principle act on internal
representations, not just outputs but never identifies an interpretable
representation of the target: it moves the whole forget activation and modifies the
weights. \textsc{SCALPEL} instead extracts a small set of sparse target-specific
features and intervenes only along their decoded directions. Comparing against RMU
therefore tests whether an interpretable, selective intervention can match a
stronger but opaque one.

Behavioral evaluation alone cannot establish that knowledge is gone: an internal
trace may remain recoverable through other prompts or light
finetuning~\citep{hong2025parametrictraces,yuan2024closerlook}. Representation-level
methods do not guarantee parametric deletion either, but they make the mechanism
explicit by causally linking behavioral change to an identified internal
intervention.

\paragraph{Mechanistic-interpretability-based unlearning.}
In transformers, a natural intervention point is the residual stream,
through which information flows between
layers~\citep{elhage2021framework,dar2022embedding,geva2022ffn,geva2023dissecting}.
Concept-decomposition methods express activations through a small set of
meaningful directions: PCA extracts high-variance orthogonal directions
and NMF factorizes non-negative activations~\citep{jolliffe2002pca,lee1999nmf};
we use NMF as a classical baseline whose components can be scored on
forget/retain data and used as intervention directions. Sparse
autoencoders offer a more flexible decomposition and have become a
standard tool for interpretable
features~\citep{cunningham2023sparse,bricken2023monosemanticity,gao2024scalingSAE,bussmann2024batchtopk,rajamanoharan2024jumprelu}.
Given an activation $h$, an SAE computes a sparse code $z(h)=E_\phi(h)$
and reconstructs $\widehat h = b_D + \sum_{q=1}^{Q} z_q(h)\,d_q$, where
$b_D$ is the decoder bias and $d_q$ the decoder direction of feature $q$;
the linear decoder associates each feature with an explicit direction in
residual space, making feature-level interventions decomposable and
inspectable.

A standard SAE unlearning pipeline trains the SAE, identifies
forget-related features from activation statistics, and modifies them at
inference. A direct intervention writes
$\widetilde h = h - \alpha \sum_{q\in S_f} z_q(h)\,d_q$
for a selected set $S_f$ and strength $\alpha$, with clamping or negative
scaling as variants~\citep{farrell2024saeunlearning}. An alternative is to
edit the sparse code and replace $h$ by the full SAE reconstruction, as in
LLM Neurosurgeon~\citep{zhou2025neurosurgeon}. That choice, however,
additionally injects the SAE's reconstruction error,
$\widetilde h - h = -\sum_{q\in S_f} z_q(h)d_q + (b_D + Dz(h) - h)$,
which in our experiments degraded retain and Wiki performance even before
any feature was removed. We therefore always intervene on the original
activation. SAE features can also constrain weight updates through
decoder-direction subspaces~\citep{wang2025saesubspace}; this suits broad,
globally separable topics, whereas our targets overlap heavily with retain
data and do not occupy a clean linear subspace. Existing MI-based unlearning methods generally reuse features learned by
generic reconstruction objectives; we instead study how to train the
extractor itself for the selectivity required by narrow forget targets.

The main limitation of these pipelines is not finding features that
activate on forget examples, but finding features that activate
\emph{specifically} on the target while remaining weak on shared
background structure. \textsc{SCALPEL} addresses this selectivity
bottleneck by adapting the feature-learning objective itself. The next
section explains why reconstruction objectives are poorly aligned with
this requirement.

\section{Energy Bias in Reconstruction-Based Extraction}
\label{sec:energy-bias}

\begin{figure*}[t]
    \centering
    \includegraphics[width=0.8\textwidth]{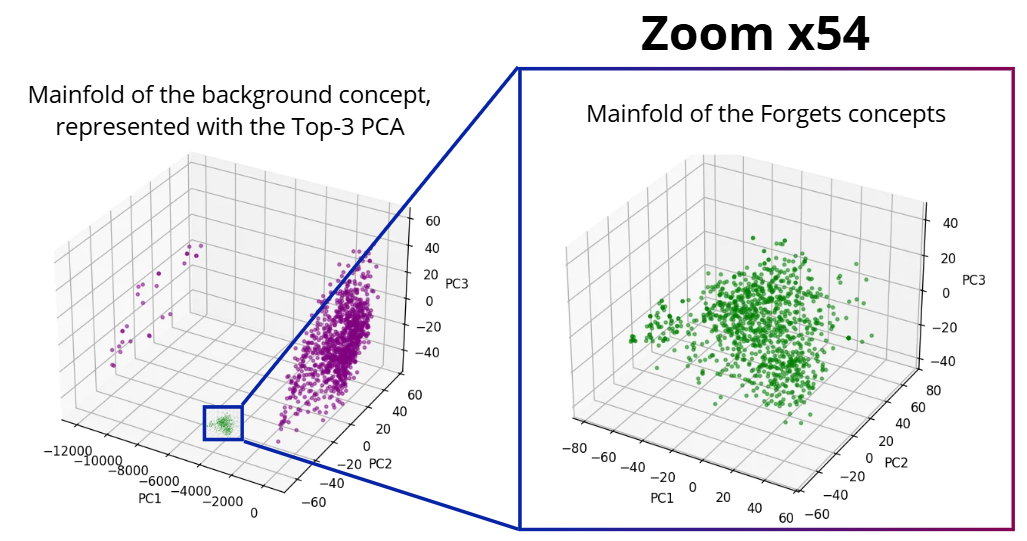}
    \caption{
    Energy comparison between dominant background structure and
    background-orthogonal forget residuals at layer $-4$ of Qwen.
    Left: background activations projected onto the first three principal
    components learned from retain and Wiki activations.
    Right: forget activations after removing these directions, shown in a
    second PCA basis fitted to the residuals. Since the panels use different
    bases, they should be compared through their scale and explained energy,
    not their coordinates. The background top-3 PCA energy is
    ${\sim}329\times$ larger.
    }
    \label{fig:energy-bias}
\end{figure*}

One important source of the selectivity bottleneck is an objective
mismatch: classical extractors optimize reconstruction, whereas targeted
unlearning requires target specificity. Following the additive
concept-manifold view of \citet{bhalla2026sparsemanifolds}, suppose an
activation decomposes as $h=\sum_{j=1}^{J}m_j$, with
$m_j\in\mathcal{M}_j$ (Mainfold j), and define its energy as
$E_j=\mathbb{E}[\lVert m_j\rVert_2^2]$. Under centered, mutually
uncorrelated concept components, the reconstruction loss of an extractor
$\Theta$ decomposes as
\[
\mathcal{L}_{\mathrm{rec}}(\Theta)
=
\sum_{j=1}^{J}E_j\bigl(1-C_j(\Theta)\bigr),
\]
where $C_j(\Theta)\in[0,1]$ is the fraction of concept $j$'s energy
captured by $\Theta$. Improving its capture by $\Delta C_j$ yields a
reconstruction gain of $E_j\Delta C_j$, so the objective assigns a larger
marginal gain to a background concept $\mathcal{M}_b$ than to a forget
concept $\mathcal{M}_f$ whenever
$E_b\Delta C_b>E_f\Delta C_f$
(Appendix~\ref{App:intro_th} and Figure \ref{fig:energy-bias}). This is a stylized account of the incentive
induced by reconstruction, not an impossibility result for trained SAEs:
a forget concept may remain important to the model while contributing
little variance. Consequently, PCA, NMF, and reconstruction-only SAEs are
more strongly rewarded for modeling high-energy structure, and a
low-energy but precise target may be blended into broader features.

We examine this on residual-stream activations from layer $-4$ of Qwen.
We fit PCA on the background activations
$H_{\mathrm{bg}}=H_r\cup H_w$, remove their top-three directions from
the forget activations, and fit a second PCA on the residuals
(Appendix~\ref{App:intro_th}). The corresponding top-three energies are
$\widehat E_{\mathrm{bg}}^{(3)}=3{,}787{,}695$ and
$\widehat E_f^{\perp\mathrm{bg},(3)}=11{,}517$, a factor of about $329$.
Although removing the background directions means that the residuals are
not a pure forget manifold, the gap is far larger than under a random
split of the same data. This is consistent with the proposed mechanism:
target-specific structure can contribute far less reconstruction energy
than the background from which it must be separated, motivating an
objective that explicitly rewards \emph{selectivity}.

\section{\textsc{SCALPEL}: Sparse Contrastive Autoencoders for Precise
Representation-Level Unlearning}
\label{sec:scalpel}

\textsc{SCALPEL} defines fine-grained forget targets, constructs multiple
semantic views, learns target-selective sparse features from cached
residual activations, and filters their decoded contribution at inference
(Figure~\ref{fig:method_overview}).

\subsection{Target granularity and multi-view construction}
\label{sec:target-granularity}

Targeted unlearning first needs a precise definition of a forget concept.
Let $\mathcal{D}_f,\mathcal{D}_r,\mathcal{D}_w$ be the forget, retain, and Wiki
datasets and $\mathcal{D}_{\mathrm{bg}}=\mathcal{D}_r\cup\mathcal{D}_w$. In TOFU,
forget and retain authors come from the same generation process and the split is
random, not semantic: they share writing style, factual structure, vocabulary, and
QA format, and their activations overlap heavily
(Appendix~\ref{app:UMAP}, Figure~\ref{fig:umap-target-granularity}). Treating the
whole forget set as one concept would push the extractor toward features of the
common TOFU format rather than the knowledge to remove. We instead define one
concept per forget author, $G_a=\{x\in\mathcal{D}_f:\operatorname{author}(x)=a\}$,
which matches the right-to-be-forgotten setting and gives several related samples
per target (sample-level granularity is discussed in
Appendix~\ref{app:limitations_full}).

For each target $G_a$ we build a set of augmented views
$\mathcal{V}_a=\{x_{a,1},\ldots,x_{a,m_a}\}$ using paraphrasing,
prompt-style rewriting with \texttt{Qwen2.5-7B-Instruct}, and
backtranslation~\citep{sennrich2016backtranslation,nllbteam2022nllb},
in the spirit of view-based contrastive
learning~\citep{chen2020simclr}. Before augmentation we extract the key
facts of the target (names, dates, titles, relations) and constrain them to remain unchanged, so the surface form varies while the
defining target facts are retained. The same recipe extends to document-level datasets such as
MUSE by treating each chunk as a target; we leave that evaluation to
future work.

\subsection{Residual-stream activation extraction}
\label{sec:activation-extraction}

We run the finetuned model and cache residual-stream activations
$h_t^\ell(x)\in\mathbb{R}^d$ at a fixed layer $\ell$, in practice the
fourth layer before the output: deep enough to carry high-level factual
information, but before the final layers specialize for next-token
prediction~\citep{yu2024hallucination,mabrok2026latentmanifolds}. We keep this layer fixed across all representation-based methods. We do
not claim that it is optimal, and a systematic study of layer selection
is left for future work. For a
dataset $\mathcal{D}_s$ and retained token positions $\mathcal{I}(x)$, we
write $\mathcal{H}_s=\{h_t^\ell(x): x\in\mathcal{D}_s,\,
t\in\mathcal{I}(x)\}$ for $s\in\{f,r,w\}$,
$\mathcal{H}_{\mathrm{bg}}=\mathcal{H}_r\cup\mathcal{H}_w$, and
$\mathcal{H}_{G_a}$ for the tokens of target $G_a$. Each activation
dimension is standardized with statistics from the training cache,
$\overline h = (h-\mu_H)\oslash\sigma_H$; both the SAE input and its
reconstruction live in this standardized space, which we found important
for robustness across model families.

\subsection{Contrastive sparse autoencoder}
\label{sec:contrastive-sae}

\paragraph{Sparse decomposition.}
All variants share the same base. The encoder maps a standardized
activation to a non-negative sparse code
$z(\overline h)=E_\phi(\overline h)\in\mathbb{R}_+^Q$, and a linear
decoder reconstructs
$\widehat{\overline h} = b_D + \sum_{q=1}^{Q} z_q(\overline h)\,d_q$.
Training combines the usual reconstruction loss
$\mathcal{L}_{\mathrm{rec}}=\mathbb{E}_{\overline h}[\lVert\overline
h-\widehat{\overline h}\rVert_2^2]$ over forget, retain, and Wiki activations with an
$\ell_1$ penalty $\mathcal{L}_{\mathrm{sparse}}=\mathbb{E}_{\overline
h}[\lVert z(\overline h)\rVert_1]$.

The raw coefficient $z_q$ alone does not measure a feature's decoded
contribution and is sensitive to reciprocal rescaling of the encoder and
decoder. We therefore use the scale-invariant norm-aware activation
$\widetilde z_q(\overline h) = z_q(\overline h)\,\lVert d_q\rVert_2
 \ge 0$,
and represent token $i=(x,t)$ by
$r_i = (\widetilde z_1(\overline h_i),\ldots,
        \widetilde z_Q(\overline h_i)) \in \mathbb{R}_+^Q$.
The contrastive objective operates directly on these token-level codes
through the cosine similarity
$\operatorname{sim}(i,j)
 = \langle r_i, r_j\rangle /
   (\lVert r_i\rVert_2 \lVert r_j\rVert_2)$;
no token averaging is done before the similarity (sample-level
aggregation is used only later, for feature scoring).

\paragraph{Token-level multi-positive InfoNCE.}
All cached activations are flattened into one pool of tokens, each
inheriting the concept identifier and split of its source sample; the
pool is shuffled and cut into minibatches each epoch, independently of
sample boundaries. For an anchor token $i$, positives $\mathcal{P}(i)$
are the other in-batch tokens sharing its concept identifier (another
token of the same sample, or a token of another view of the same
target), and negatives $\mathcal{N}(i)$ are in-batch tokens with a
different concept identifier that satisfy the split-specific contrastive
mask; depending on the variant, negatives may include retain tokens or Wiki tokens. Anchors with an empty positive
or negative set are dropped. With temperature $\tau>0$ and
$\kappa(i,j)=\exp(\operatorname{sim}(i,j)/\tau)$, define the positive and
negative masses
$A_i^+=\sum_{j\in\mathcal{P}(i)}\kappa(i,j)$ and
$A_i^-=\sum_{j\in\mathcal{N}(i)}\kappa(i,j)$, and the multi-positive
InfoNCE loss~\citep{oord2018infonce,khosla2020supcon}
\[
\ell_{\mathrm{NCE}}(i)
= -\log\frac{A_i^+}{A_i^+ + A_i^-},
\qquad
\mathcal{L}_{\mathrm{NCE}}
= \frac{1}{|\mathcal{I}_{\mathrm{valid}}|}
  \sum_{i\in\mathcal{I}_{\mathrm{valid}}}
  \ell_{\mathrm{NCE}}(i).
\]
This pulls together token codes of the same target and pushes them away
from other concepts and from background data. Crucially, negatives come
from retain data rather than from other forget examples, so related
forget samples are never pushed apart.

\subsection{\textsc{SCALPEL}}
\label{sec:scalpel-variants}

\paragraph{\textsc{SCALPEL}.}
Forget and background samples share a lot of structure, and forcing one latent code
to serve both reconstruction and separation can be restrictive. Inspired by
contrastive SAEs of ~\citep{poupart2024contrastiveSAE}, we split the code into two blocks
produced by two encoders, $z(\overline h)=[z^{\mathrm{com}}(\overline h),
z^{\mathrm{diff}}(\overline h)]$, decoded jointly. The common block captures
structure shared across samples ($\mathcal{L}_{\mathrm{com}}$); the target-specific
block contrasts views of the same forget target against retain and Wiki examples
($\mathcal{L}_{\mathrm{diff}}$). The objective is
$\mathcal{L}_{\mathrm{split}}=\mathcal{L}_{\mathrm{rec}}
+\lambda\mathcal{L}_{\mathrm{sparse}}
+\beta_{\mathrm{com}}\mathcal{L}_{\mathrm{com}}
+\beta_{\mathrm{diff}}\mathcal{L}_{\mathrm{diff}}$, and only features of the
target-specific block are eligible for intervention.

\subsection{Feature selection and residual-stream intervention}
\label{sec:feature-selection}

After training, we score each feature by how active it is on one target
relative to the background. With
$
\mu_q(G_a)
=
\mathbb{E}_{\overline h\in\overline{\mathcal H}_{G_a}}
[\widetilde z_q(\overline h)],
$ and $
\mu_q(\mathrm{bg})
=
\mathbb{E}_{\overline h\in\overline{\mathcal H}_{\mathrm{bg}}}
[\widetilde z_q(\overline h)],
$
the target-selectivity score is
\begin{equation}
s_q(G_a)
=
\frac{\mu_q(G_a)}
{\mu_q(\mathrm{bg})+\varepsilon}.
\label{eq:score}
\end{equation}
For each target, we select
$S_{G_a}=\{q:s_q(G_a)\geq s_{\min}\}$, with a top-feature fallback if
the set is empty. For a forget request containing several targets, the
fixed global intervention set is
$
S_f=\bigcup_{a\in\mathcal A_f}S_{G_a},
$
where shared features are included only once.

The same set $S_f$ is applied to every input, without access to author identity or
forget/retain labels; its effect stays input-dependent, since a selected feature is
subtracted only when the encoder activates it. As the SAE operates in standardized
space, we compute $\overline h^\ell=(h^\ell-\mu_H)\oslash\sigma_H$ and map the
decoded contribution back:
\begin{equation}
\widetilde h^\ell = h^\ell - \alpha\sum_{q\in S_f} z_q(\overline
h^\ell)\bigl(\sigma_H\odot d_q\bigr).
\label{eq:intervention}
\end{equation}
Modifying the original activation rather than replacing it with the full
reconstruction avoids injecting reconstruction error, and each component remains a
sparse activation times a decoder direction.

\section{Theoretical Analysis}
\label{sec:theory}

We formalize two properties of \textsc{SCALPEL}: adding contrastive
learning to reconstruction and sparsity favors more target-selective
features, and selecting such features limits their expected effect on
background activations. Complete assumptions and proofs are given in
Appendices~\ref{app:sparse_contrastive} and~\ref{App:score_unlearing}.

\paragraph{Contrastive learning promotes target-selective features.}
Write
$\mathcal L_{\mathrm{SAE}}=\mathcal L_{\mathrm{rec}}+
\lambda\mathcal L_{\mathrm{sparse}}$ and
$\mathcal L_{\mathrm{SCALPEL}}=\mathcal L_{\mathrm{SAE}}+
\beta\mathcal L_{\mathrm{NCE}}$.
Reconstruction maintains an informative representation and sparsity
concentrates it on few coordinates, but neither explicitly rewards
target--background separation. The contrastive term provides this
missing signal.

Let $\widetilde z_q(h)=z_q(h)\lVert d_q\rVert_2$ be the norm-aware code
used by InfoNCE and
$u_i=\widetilde z(h_i)/\lVert\widetilde z(h_i)\rVert_2$ its normalized
representation.

\begin{proposition}[Contrastive regularization favors target selectivity]
\label{prop:contrastive-selectivity}
Assume non-negative codes, controlled decoder scales, and that a sparse
target--background separation is compatible with comparable
reconstruction and sparsity cost. Among such representations, adding
$\mathcal L_{\mathrm{NCE}}$ favors
$\langle u_i,u_j\rangle\rightarrow1$ for same-target positives and
$\langle u_i,u_j\rangle\rightarrow0$ for background negatives.
More precisely, if the contrastive positive-to-negative mass ratio is
within a factor $1+\delta$ of its optimum, then
\[\frac{1}{|N(i)|}\sum_{j\in N(i)}\langle u_i,u_j\rangle\leq\tau\delta \qquad
\text{and} \qquad
\frac{1}{|P(i)|}\sum_{j\in P(i)}\lVert u_i-u_j\rVert_2^2
\leq\frac{2\delta}{(1+\delta)(1-e^{-1/\tau})}\]
\end{proposition}

Thus, as the contrastive objective improves, views of the same target
share increasingly similar sparse coordinates while their overlap with
background codes decreases. Combined with sparsity, this favors features
that are active on a target and weak on the background, precisely the
structure measured by the selectivity score $s_q(G_a)$ in
Eq.~\eqref{eq:score}. Different forget targets may still share features
when they are not contrasted against each other.

\paragraph{Selectivity controls background perturbation.}
In standardized activation space, let
$\delta_{G_a}(\overline h)=
\alpha\sum_{q\in S_{G_a}}z_q(\overline h)d_q$.

\begin{proposition}[Selectivity controls background perturbation]
\label{prop:retain-damage}
If every selected feature satisfies $s_q(G_a)\geq s_{\min}$, then
\[\mathbb{E}_{\overline h\sim\overline{\mathcal H}_{\mathrm{bg}}}
[\lVert\delta_{G_a}(\overline h)\rVert_2]
\leq\alpha\sum_{q\in S_{G_a}}\mu_q(\mathrm{bg})
\leq\frac{\alpha}{s_{\min}}
\sum_{q\in S_{G_a}}\mu_q(G_a)
\]
\end{proposition}

The bound controls expected representation-space perturbation, not
downstream behavior or worst-case effects. It also ignores the
orientation of $d_q$, so similarly scored features may have different
causal effects, as examined in Section~\ref{sec:score-results}.

Together, the results explain the intended mechanism: reconstruction
preserves information, sparsity concentrates it, contrastive learning
organizes the sparse code according to the target concepts, and the
resulting selectivity limits the expected background footprint of the
intervention.
\section{Experiments}
\label{sec:experiments}

We run one main benchmark and two diagnostics. The benchmark compares
representation-level unlearning performance against classical,
mechanistic-interpretability, and optimization-based baselines across
three model families. The diagnostics test whether contrastive training
produces more target-selective sparse-code geometry and whether the
resulting selectivity score predicts intervention quality.

\paragraph{Models and reference checkpoints.}
We use \texttt{Qwen2.5-1.5B-Instruct},
\texttt{Llama-3.2-1B-Instruct}, and
\texttt{Gemma-2-2B-it}~\citep{yang2024qwen25,grattafiori2024llama3,gemmateam2024gemma2}.
For each model family and seed, we produce two references: a
\emph{full} model $\theta_{\mathrm{full}}$ finetuned on
$\mathcal{D}_f\cup\mathcal{D}_r$, which is the starting point of every
unlearning method, and a \emph{retain-only} model
$\theta_{\mathrm{ret}}$ finetuned on $\mathcal{D}_r$ alone, which
approximates the desired endpoint. Within each seed, all methods start
from the same full checkpoint. Augmented views are generated separately
with \texttt{Qwen2.5-7B-Instruct}.

\paragraph{Dataset and augmentation.}
We evaluate on TOFU~\citep{maini2024tofu}, using Wiki data
$\mathcal{D}_w$ to measure general-knowledge preservation. For each
forget author, we construct augmented views using 16 templates, including
question--answer reformulations, biographies, Wiki-style passages,
paraphrases, multi-hop questions, and adversarially phrased
instructions~\citep{zou2023universal,dahal2026gone}. We then
backtranslate them through French, Spanish, and German with
NLLB-200~\citep{sennrich2016backtranslation,nllbteam2022nllb}.
The transformations vary wording and format while preserving the names,
dates, titles, and relations defining the target.

\paragraph{Compared methods.}
We compare \textbf{NMF}~\citep{lee1999nmf}; a \textbf{standard SAE} trained with
reconstruction and sparsity
only~\citep{cunningham2023sparse,farrell2024saeunlearning}; \textbf{contrastive
SAEs} with a single shared code, using one- or two-layer encoders;
\textbf{\textsc{SCALPEL}}; \textbf{Gradient
Difference}~\citep{bu2024ngdiff}; and \textbf{RMU}~\citep{huutien2024adaptiveRMU}.
All representation-based methods use layer $L-4$, the same token positions, and the
same selection and intervention procedure whenever the decomposition permits it, so
the comparison isolates the effect of the learned representation.

\paragraph{Protocol.}
Each method is tuned independently for every model family using
Optuna~\citep{akiba2019optuna}, with 25 trials per model--method pair
and the search spaces reported in Appendix~\ref{app:hpo}. The selected
configuration is frozen and rerun over five independent end-to-end seeds.
Each seed repeats reference-model finetuning, method training, feature
selection, and evaluation with
OpenUnlearning~\citep{dorna2025openunlearning}.

\paragraph{Metrics and trade-off curves.}
Our main trade-off compares TOFU model utility with forget-set
ground-truth answer probability~\citep{maini2024tofu}. We additionally
report Q--A probability and ROUGE on forget, retain, and Wiki data,
together with forget quality, truth ratio, privacy leakage, extraction
strength, and paraphrase-based evaluation
\citep{lin2004rouge,maini2024tofu,shi2024muse,
carlini2021extracting,dorna2025openunlearning}.
For methods exposing an intervention strength $\alpha$, we trace the
forget--retain trade-off using adaptive bisection in $\log\alpha$ and
interpolate the five seed-specific curves onto a common grid
(Appendix~\ref{app:alpha_search}). Complete multi-metric trade-offs are
reported in
Appendix~\ref{app:additional_benchmarks_tradeoff_curves}.

\paragraph{Diagnostics.}
The \emph{contrastive-geometry diagnostic}
(Section~\ref{sec:geometry-results}) compares the standard SAE,
contrastive SAE, and \textsc{SCALPEL}-Split across the five Qwen benchmark
seeds. On held-out activations, we normalize the norm-aware codes as
$u=\widetilde z/\lVert\widetilde z\rVert_2$. For each forget target
$G_a$, we measure the within-target similarity
$C_{\mathrm{pos}}(G_a)=
\mathbb{E}_{i,j\in G_a}[\langle u_i,u_j\rangle]$
and the target--background similarity
$C_{\mathrm{bg}}(G_a)=
\mathbb{E}_{i\in G_a,j\in\mathrm{bg}}[\langle u_i,u_j\rangle]$.
Higher $C_{\mathrm{pos}}$ and lower $C_{\mathrm{bg}}$ indicate a more
target-selective representation. We then compare feature selectivity
across methods to test whether this geometric separation translates into
more selective sparse features.

The \emph{score-quality diagnostic}
(Section~\ref{sec:score-results}) considers the top-5 feature--target
pairs for each of ten authors, giving 50 selections. Each pair is
evaluated alone by sweeping $\alpha$ and tracing its
$(P_{\mathrm{forget}},P_{\mathrm{retain}})$ trajectory, which we compare
across selectivity-score quantile bins.

\section{Results}
\label{sec:results}

\subsection{Main unlearning benchmark}
\label{sec:benchmark-results}

\begin{figure*}[t]
    \centering
    \includegraphics[width=\textwidth]
    {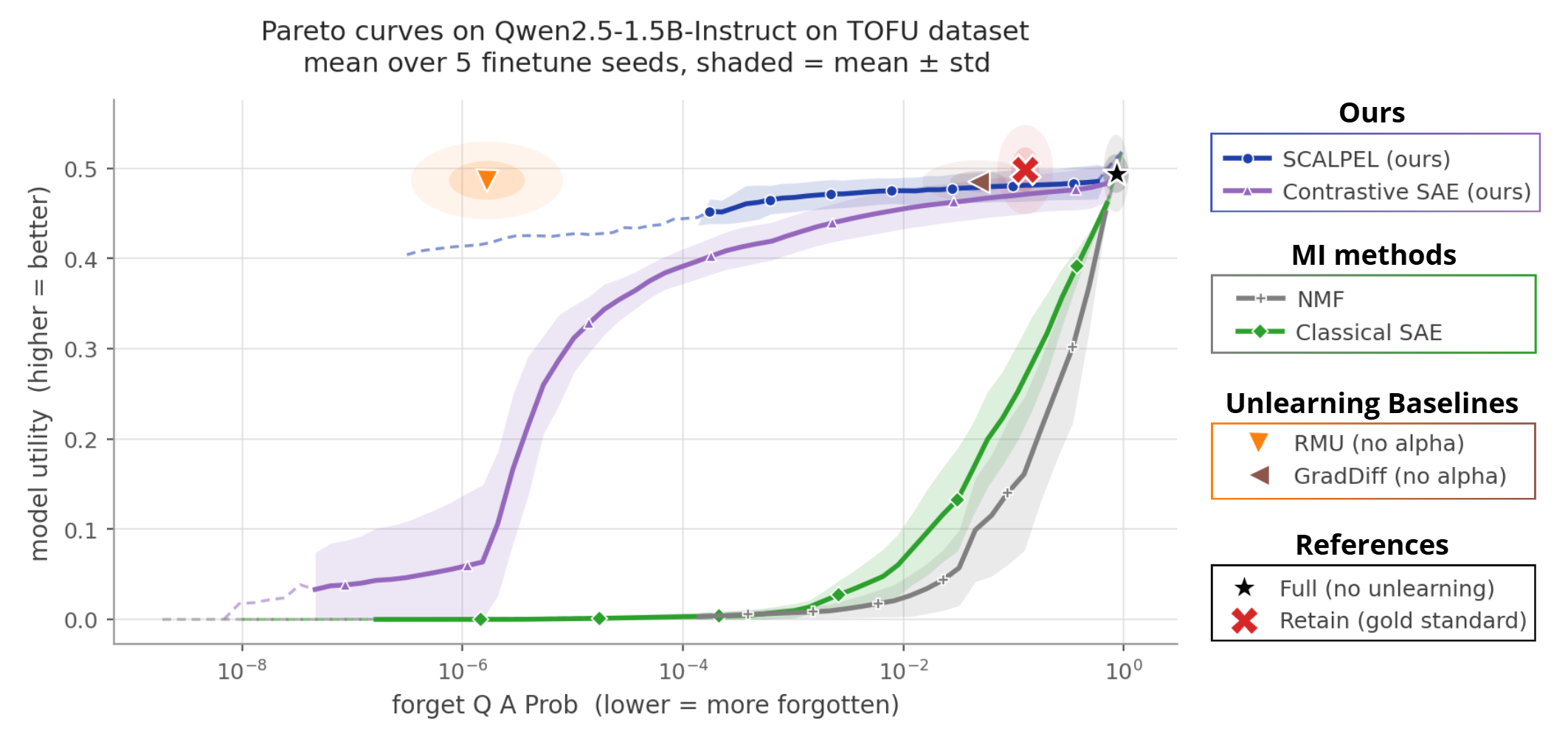}
    \caption{
    Unlearning trade-off on TOFU forget10 for
    \texttt{Qwen2.5-1.5B-Instruct}, after method-specific
    hyperparameter optimization and aggregation over five independent
    end-to-end seeds. Each curve varies the intervention strength:
    lower forget probability indicates stronger forgetting, while higher
    model utility indicates better preservation. Results for Llama and
    Gemma are reported in
    Appendix~\ref{app:additional_benchmarks} are similar. For more metrics look at Appendix~\ref{app:additional_benchmarks_tradeoff_curves}.
    }
    \label{fig:benchmark}
\end{figure*}

Figure~\ref{fig:benchmark} shows the main benchmark on Qwen; Llama and Gemma behave consistently and are reported in Appendix~\ref{app:additional_benchmarks}. Three patterns stand out. First, NMF and the standard SAE are poorly selective: reductions in forget probability come with substantial loss in model utility, as predicted by the objective mismatch of Section~\ref{sec:energy-bias} their features fire on forget examples but also capture structure shared with retain and Wiki data.

Second, contrastive feature learning substantially improves the
trade-off. The plain contrastive SAEs already outperform the
reconstruction-only SAE, while the main \textsc{SCALPEL} variants lie on
or close to the Pareto front over the practically relevant range. This
supports our central claim that adapting the extraction objective to
target selectivity produces more effective representation-level
interventions.

Third, RMU and Gradient Difference remain strong baselines, but
\textsc{SCALPEL} achieves competitive trade-offs while keeping the base
model fixed and expressing the intervention through explicit sparse
feature activations and decoder directions. Corresponding experiments on
Llama and Gemma are provided in
Appendix~\ref{app:additional_benchmarks}.

\subsection{Contrastive learning increases feature selectivity}
\label{sec:geometry-results}
Figure~\ref{fig:contrastive-geometry} tests the geometric prediction of
Proposition~\ref{prop:contrastive-selectivity}. The reconstruction-only SAE
lies close to $C_{\mathrm{pos}}=C_{\mathrm{bg}}$, with a separation gap
$\Delta C=C_{\mathrm{pos}}-C_{\mathrm{bg}}\approx-0.01$, indicating little
target--background separation. Adding the contrastive objective increases
within-target similarity and raises the gap to $\Delta C\approx0.27$.
\textsc{SCALPEL}-Split further reduces target--background similarity,
reaching $\Delta C\approx0.43$. The representation therefore moves
progressively toward the high-$C_{\mathrm{pos}}$, low-$C_{\mathrm{bg}}$
regime favored by the sparse contrastive objective.

This geometric separation is also reflected at the feature level.
Appendix~\ref{app:contrastive-selectivity} compares the top-5 selectivity
scores of the same models across authors and seeds. Mean top-5 selectivity
increases from approximately $1.9$ for the reconstruction-only SAE to
$4.0$ for the contrastive SAE, and to $24.5$ for
\textsc{SCALPEL}-Split. Together, these results support the proposed role
of contrastive learning: reconstruction and sparsity produce a compact
representation, while the contrastive objective organizes its features
around the target concepts and away from background structure.

\subsection{The selectivity score predicts unlearning quality}
\label{sec:score-results}

\begin{figure*}[t]
    \centering

    \begin{subfigure}[t]{0.49\textwidth}
        \centering
        \includegraphics[width=\linewidth]
        {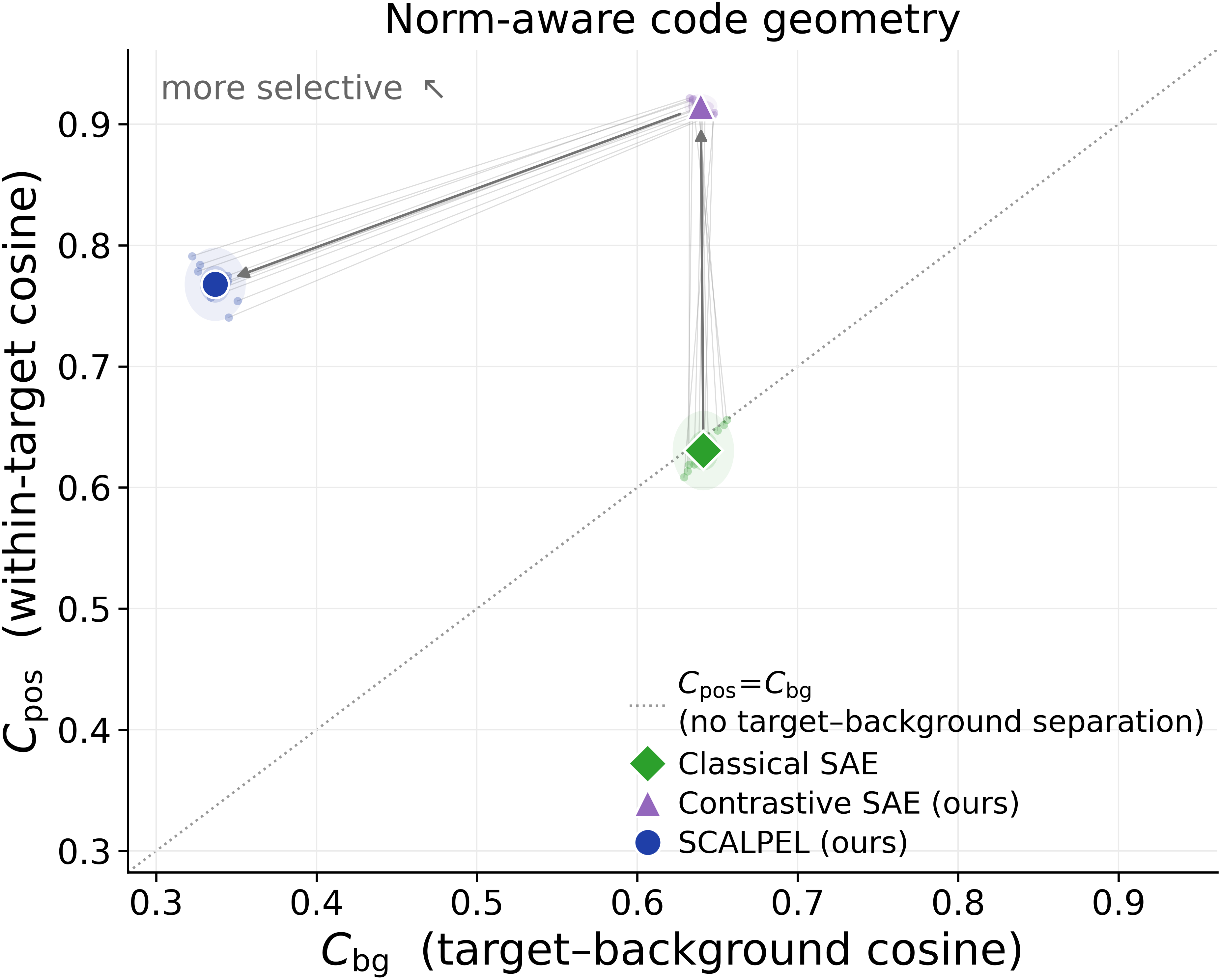}
        \caption{
        Norm-aware sparse-code geometry across the five Qwen benchmark seeds.
        For each method, $C_{\mathrm{pos}}$ measures within-target cosine
        similarity and $C_{\mathrm{bg}}$ target--background similarity on held-out
        activations. More selective representations lie toward the top-left.
        Adding contrastive learning increases within-target alignment, while
        \textsc{SCALPEL}-Split further reduces background overlap.
        }
        \label{fig:contrastive-geometry}
    \end{subfigure}
    \hfill
    \begin{subfigure}[t]{0.49\textwidth}
        \centering
        \includegraphics[width=\linewidth]
        {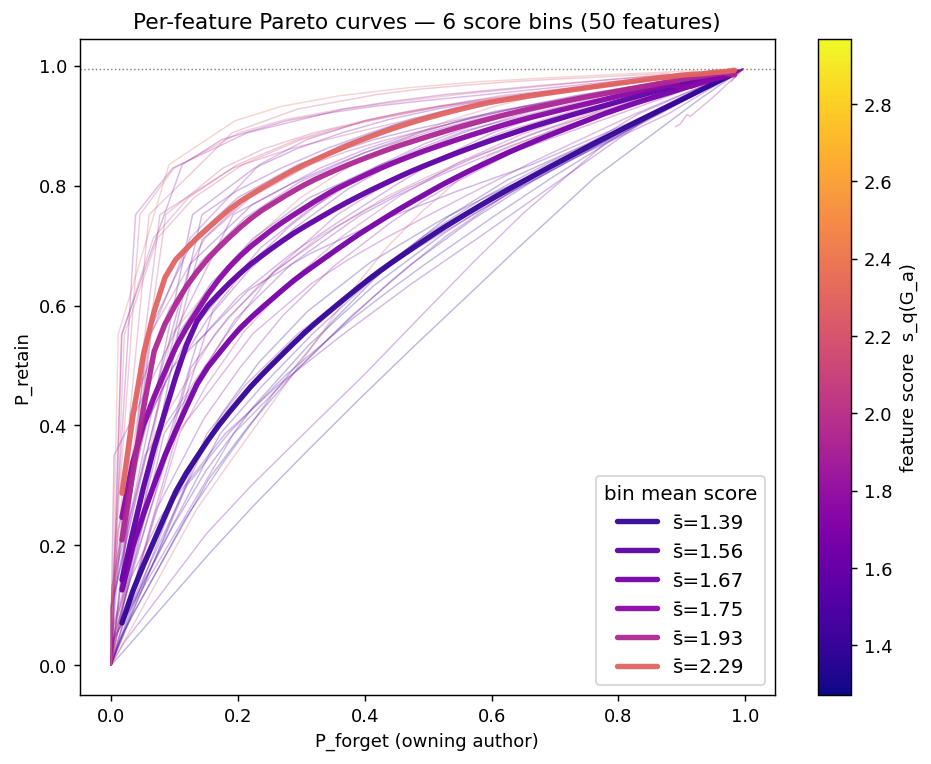}
        \caption{
        Single-feature unlearning curves grouped by selectivity-score
        quantile. Thin curves represent individual author--feature pairs;
        thick curves show the average within each bin. Higher-score features
        generally provide better forget--retain trade-offs.
        }
        \label{fig:score-pareto}
    \end{subfigure}

   \caption{Diagnostics linking representation geometry, feature selectivity, and
    intervention quality for \textsc{SCALPEL}.
    Left: contrastive training increases within-target alignment and reduces
    target--background overlap in the norm-aware sparse code.
    Right: features with higher selectivity scores generally achieve better
    single-feature forget--retain trade-offs.
    }
    \label{fig:scalpel-diagnostics}
\end{figure*}
Figure~\ref{fig:score-pareto} tests whether Eq.~\eqref{eq:score} predicts the
quality of individual interventions. Each author--feature pair is evaluated in
isolation while varying its intervention strength. In aggregate, higher-score
bins preserve more retain probability at fixed forget probability, or achieve
stronger forgetting at comparable retain probability. The selectivity score is
therefore associated not only with feature selection but also with downstream
intervention quality. Together with the contrastive-geometry diagnostic, this
connects the learned representation to feature selectivity and, in turn, to
intervention quality. The ordering remains imperfect at the level of individual
features: the score captures activation magnitude and target--background
selectivity, but ignores the orientation of $d_q$ in the local activation
geometry. A geometry-aware score would require substantially more expensive
local manifold estimation (Appendix~\ref{app:limitations_full}).

\section{Conclusion and Limitations}
\label{sec:conclusion}

We proposed a first approach for connecting mechanistic interpretability
with fine-grained unlearning of specific information. Our study shows that
reconstruction-based concept extractors can lack the selectivity required
to isolate a narrow forget target from the background knowledge that should
be preserved. We identify one mechanism behind this behavior:
reconstruction rewards captured energy, while target-specific information
may correspond to a low-energy component entangled with much larger
background structure; at the studied layer, we measure a
${\sim}329\times$ energy gap. \textsc{SCALPEL} addresses this problem by
adapting the feature-learning objective rather than the intervention
pipeline, contrasting multiple views of each target against background
data to learn sparse, target-selective features. Across Qwen, Llama, and
Gemma on TOFU, \textsc{SCALPEL} substantially improves over NMF and
reconstruction-only SAE interventions and is competitive with RMU and
Gradient Difference. Our theoretical analysis and diagnostics further
show how contrastive learning promotes target-selective representations
and how feature selectivity relates to intervention quality.

\paragraph{Limitations.}
Our evaluation covers TOFU and three relatively small models; larger models, less
controlled datasets, and document-, fact-, or sample-level targets remain untested.
The theoretical analysis is local, and the perturbation bound controls
representation-space magnitude rather than model behavior. Most importantly,
\textsc{SCALPEL} does not erase the parametric trace: it applies a persistent
internal filter to selected feature contributions whenever they activate, so our
results establish selective behavioral and representation-level unlearning under
intervention, not certified deletion. Using the extracted features to guide permanent
parameter updates while preserving their selectivity is the natural next step. A
fuller discussion is in Appendix~\ref{app:limitations_full}.

\bibliography{iclr2026_conference}

@misc{dorna2025openunlearning,
  title        = {OpenUnlearning: Accelerating LLM Unlearning via Unified Benchmarking of Methods and Metrics},
  author       = {Dorna, V. and others},
  year         = {2025},
  eprint       = {2506.12618},
  archivePrefix = {arXiv},
  primaryClass = {cs.CL}
}

@misc{maini2024tofu,
  title        = {{TOFU}: A Task of Fictitious Unlearning for {LLM}s},
  author       = {Maini, Pratyush and Feng, Zhili and Schwarzschild, Avi and Lipton, Zachary C. and Kolter, J. Zico},
  year         = {2024},
  eprint       = {2401.06121},
  archivePrefix = {arXiv},
  primaryClass = {cs.LG}
}

@misc{shi2024muse,
  title        = {{MUSE}: Machine Unlearning Six-Way Evaluation for Language Models},
  author       = {Shi, Weijia and Lee, Jaechan and Huang, Yangsibo and Malladi, Sadhika and Zhao, Jieyu and Holtzman, Ari and Liu, Daogao and Zettlemoyer, Luke and Smith, Noah A. and Zhang, Chiyuan},
  year         = {2024},
  eprint       = {2407.06460},
  archivePrefix = {arXiv},
  primaryClass = {cs.CL}
}

@misc{yang2024qwen25,
  title        = {{Qwen2.5} Technical Report},
  author       = {Yang, An and Yang, Baosong and Zhang, Beichen and Hui, Binyuan and Zheng, Bo and Yu, Bowen and Li, Chengyuan and Liu, Dayiheng and Huang, Fei and Wei, Haoran and Lin, Huan and Yang, Jian and Tu, Jianhong and Zhang, Jianwei and Yang, Jianxin and Yang, Jiaxi and Zhou, Jingren and Lin, Junyang and Dang, Kai and Lu, Keming and Bao, Keqin and Yang, Kexin and Yu, Le and Li, Mei and Xue, Mingfeng and Zhang, Pei and Zhu, Qin and Men, Rui and Lin, Runji and Lin, Tianhao and Tang, Tianyi and Xia, Tingyu and Ren, Xingzhang and Ren, Xuancheng and Fan, Yang and Su, Yang and Zhang, Yichang and Wan, Yu and Liu, Yuqiong and Cui, Zeyu and Zhang, Zhenru and Qiu, Zihan},
  year         = {2024},
  eprint       = {2412.15115},
  archivePrefix = {arXiv},
  primaryClass = {cs.CL}
}

@misc{grattafiori2024llama3,
  title        = {The {Llama} 3 Herd of Models},
  author       = {Grattafiori, Aaron and others},
  year         = {2024},
  eprint       = {2407.21783},
  archivePrefix = {arXiv},
  primaryClass = {cs.AI}
}

@misc{gemmateam2024gemma2,
  title        = {{Gemma 2}: Improving Open Language Models at a Practical Size},
  author       = {{Gemma Team}},
  year         = {2024},
  eprint       = {2408.00118},
  archivePrefix = {arXiv},
  primaryClass = {cs.CL}
}

@misc{nllbteam2022nllb,
  title        = {No Language Left Behind: Scaling Human-Centered Machine Translation},
  author       = {{NLLB Team} and Costa-juss{\`a}, Marta R. and Cross, James and {\c{C}}elebi, Onur and Elbayad, Maha and Heafield, Kenneth and Heffernan, Kevin and Kalbassi, Elahe and Lam, Janice and Licht, Daniel and Maillard, Jean and Sun, Anna and Wang, Skyler and Wenzek, Guillaume and Youngblood, Al and Akula, Bapi and Barrault, Loic and Gonzalez, Gabriel Mejia and Hansanti, Prangthip and Hoffman, John and Jarrett, Semarley and Sadagopan, Kaushik Ram and Rowe, Dirk and Spruit, Shannon and Tran, Chau and Andrews, Pierre and Ayan, Necip Fazil and Bhosale, Shruti and Edunov, Sergey and Fan, Angela and Gao, Cynthia and Goswami, Vedanuj and Guzm{\'a}n, Francisco and Koehn, Philipp and Mourachko, Alexandre and Ropers, Christophe and Saleem, Safiyyah and Schwenk, Holger and Wang, Jeff},
  year         = {2022},
  eprint       = {2207.04672},
  archivePrefix = {arXiv},
  primaryClass = {cs.CL}
}

@misc{dar2022embedding,
  title        = {Analyzing Transformers in Embedding Space},
  author       = {Dar, Guy and Geva, Mor and Gupta, Ankit and Berant, Jonathan},
  year         = {2022},
  eprint       = {2209.02535},
  archivePrefix = {arXiv},
  primaryClass = {cs.CL}
}

@inproceedings{meng2022rome,
  title        = {Locating and Editing Factual Associations in {GPT}},
  author       = {Meng, Kevin and Bau, David and Andonian, Alex and Belinkov, Yonatan},
  booktitle    = {Advances in Neural Information Processing Systems},
  year         = {2022}
}

@inproceedings{jourdan2023cockatiel,
  title        = {{COCKATIEL}: {CO}ntinuous {C}oncept ran{K}ed {AT}tribution with {I}nterpretable {EL}ements for Explaining Neural Net Classifiers on {NLP} Tasks},
  author       = {Jourdan, Fanny and Picard, Agustin and Fel, Thomas and Risser, Laurent and Loubes, Jean Michel and Asher, Nicholas},
  booktitle    = {Findings of the Association for Computational Linguistics: ACL 2023},
  year         = {2023},
  pages        = {5120--5136}
}

@inproceedings{fel2023holistic,
  title        = {A Holistic Approach to Unifying Automatic Concept Extraction and Concept Importance Estimation},
  author       = {Fel, Thomas and Boutin, Victor and Moayeri, Mazda and Cad{\`e}ne, R{\'e}mi and Bethune, Louis and And{\'e}ol, L{\'e}o and Chalvidal, Mathieu and Serre, Thomas},
  booktitle    = {Advances in Neural Information Processing Systems},
  year         = {2023}
}

@misc{cunningham2023sparse,
  title        = {Sparse Autoencoders Find Highly Interpretable Features in Language Models},
  author       = {Cunningham, Hoagy and Ewart, Aidan and Riggs, Logan and Huben, Robert and Sharkey, Lee},
  year         = {2023},
  eprint       = {2309.08600},
  archivePrefix = {arXiv},
  primaryClass = {cs.LG}
}

@misc{bricken2023monosemanticity,
  title        = {Towards Monosemanticity: Decomposing Language Models with Dictionary Learning},
  author       = {Bricken, Trenton and Templeton, Adly and Batson, Joshua and Chen, Brian and Jermyn, Adam and Conerly, Tom and Turner, Nicholas L. and Anil, Cem and Denison, Carson and Askell, Amanda and Lasenby, Robert and Wu, Yifan and Kravec, Shauna and Schiefer, Nicholas and Maxwell, Tim and Joseph, Nicholas and Hatfield-Dodds, Zac and Tamkin, Alex and Nguyen, Karina and McLean, Brayden and Burke, Josiah E. and Hume, Tristan and Carter, Shan and Henighan, Tom and Olah, Christopher},
  year         = {2023},
  howpublished = {Transformer Circuits Thread},
  note         = {\url{https://transformer-circuits.pub/2023/monosemantic-features}}
}

@misc{gao2024scalingSAE,
  title        = {Scaling and Evaluating Sparse Autoencoders},
  author       = {Gao, Leo and {Dupr{\'e} la Tour}, Tom and Tillman, Henk and Goh, Gabriel and Troll, Rajan and Radford, Alec and Sutskever, Ilya and Leike, Jan and Wu, Jeffrey},
  year         = {2024},
  eprint       = {2406.04093},
  archivePrefix = {arXiv},
  primaryClass = {cs.LG}
}

@misc{bussmann2024batchtopk,
  title        = {{BatchTopK} Sparse Autoencoders},
  author       = {Bussmann, Bart and Leask, Patrick and Nanda, Neel},
  year         = {2024},
  eprint       = {2412.06410},
  archivePrefix = {arXiv},
  primaryClass = {cs.LG}
}

@misc{rajamanoharan2024jumprelu,
  title        = {Jumping Ahead: Improving Reconstruction Fidelity with {JumpReLU} Sparse Autoencoders},
  author       = {Rajamanoharan, Senthooran and Lieberum, Tom and Sonnerat, Nicolas and Conmy, Arthur and Varma, Vikrant and Kram{\'a}r, J{\'a}nos and Nanda, Neel},
  year         = {2024},
  eprint       = {2407.14435},
  archivePrefix = {arXiv},
  primaryClass = {cs.LG}
}

@misc{farrell2024saeunlearning,
  title        = {Applying Sparse Autoencoders to Unlearn Knowledge in Language Models},
  author       = {Farrell, Eoin and Lau, Yeu-Tong and Conmy, Arthur},
  year         = {2024},
  eprint       = {2410.19278},
  archivePrefix = {arXiv},
  primaryClass = {cs.LG}
}

@misc{patil2024neurosurgeon,
  title        = {Targeted Knowledge Removal in {LLM}s Using Sparse Autoencoders},
  author       = {Patil, Kaustubh},
  year         = {2024},
  howpublished = {OpenReview},
  note         = {\url{https://openreview.net/forum?id=aeQeXlG2Pw}}
}

@misc{poupart2024contrastiveSAE,
  title        = {Contrastive Sparse Autoencoders for Interpreting Planning of Chess-Playing Agents},
  author       = {Poupart, Yoann},
  year         = {2024},
  eprint       = {2406.04028},
  archivePrefix = {arXiv},
  primaryClass = {cs.LG}
}

@misc{eldan2023harrypotter,
  title        = {Who's Harry Potter? Approximate Unlearning in {LLM}s},
  author       = {Eldan, Ronen and Russinovich, Mark},
  year         = {2023},
  eprint       = {2310.02238},
  archivePrefix = {arXiv},
  primaryClass = {cs.CL}
}

@misc{yuan2024closerlook,
  title        = {A Closer Look at Machine Unlearning for Large Language Models},
  author       = {Yuan, Xiaojian and Pang, Tianyu and Du, Chao and Chen, Kejiang and Zhang, Weiming and Lin, Min},
  year         = {2024},
  eprint       = {2410.08109},
  archivePrefix = {arXiv},
  primaryClass = {cs.CL}
}

@misc{bu2024ngdiff,
  title        = {Unlearning as Multi-Task Optimization: A Normalized Gradient Difference Approach with an Adaptive Learning Rate},
  author       = {Bu, Zhiqi and Jin, Xiaomeng and Vinzamuri, Bhanukiran and Ramakrishna, Anil and Chang, Kai-Wei and Cevher, Volkan and Hong, Mingyi},
  year         = {2024},
  eprint       = {2410.22086},
  archivePrefix = {arXiv},
  primaryClass = {cs.LG}
}

@misc{huutien2024adaptiveRMU,
  title        = {On Effects of Steering Latent Representation for Large Language Model Unlearning},
  author       = {Huu-Tien, Dang and Pham, Trung-Tin and Thanh-Tung, Hoang and Inoue, Naoya},
  year         = {2024},
  eprint       = {2408.06223},
  archivePrefix = {arXiv},
  primaryClass = {cs.CL}
}

@misc{elhage2021framework,
  title        = {A Mathematical Framework for Transformer Circuits},
  author       = {Elhage, Nelson and Nanda, Neel and Olsson, Catherine and others},
  year         = {2021},
  howpublished = {Transformer Circuits Thread},
  note         = {\url{https://transformer-circuits.pub/2021/framework/index.html}}
}

@inproceedings{geva2022ffn,
  title        = {Transformer Feed-Forward Layers Build Predictions by Promoting Concepts in the Vocabulary Space},
  author       = {Geva, Mor and Caciularu, Avi and Wang, Kevin Ro and Goldberg, Yoav},
  booktitle    = {Proceedings of the 2022 Conference on Empirical Methods in Natural Language Processing},
  pages        = {30--45},
  year         = {2022},
  publisher    = {Association for Computational Linguistics},
  doi          = {10.18653/v1/2022.emnlp-main.3}
}

@inproceedings{geva2023dissecting,
  title        = {Dissecting Recall of Factual Associations in Auto-Regressive Language Models},
  author       = {Geva, Mor and Bastings, Jasmijn and Filippova, Katja and Globerson, Amir},
  booktitle    = {Proceedings of the 2023 Conference on Empirical Methods in Natural Language Processing},
  pages        = {12216--12235},
  year         = {2023},
  publisher    = {Association for Computational Linguistics},
  doi          = {10.18653/v1/2023.emnlp-main.751}
}

@misc{yu2024hallucination,
  title        = {Mechanisms of Non-Factual Hallucinations in Language Models},
  author       = {Yu, Lei and Cao, Meng and Cheung, Jackie Chi Kit and Dong, Yue},
  year         = {2024},
  eprint       = {2403.18167},
  archivePrefix = {arXiv},
  primaryClass = {cs.CL}
}

@inproceedings{hong2025parametrictraces,
  title        = {Intrinsic Test of Unlearning Using Parametric Knowledge Traces},
  author       = {Hong, Yihuai and Yu, Lei and Yang, Haiqin and Ravfogel, Shauli and Geva, Mor},
  booktitle    = {Proceedings of the 2025 Conference on Empirical Methods in Natural Language Processing},
  pages        = {19513--19535},
  year         = {2025},
  publisher    = {Association for Computational Linguistics},
  doi          = {10.18653/v1/2025.emnlp-main.985}
}

@inproceedings{wang2025saesubspace,
  title        = {Model Unlearning via Sparse Autoencoder Subspace Guided Projections},
  author       = {Wang, Xu and Li, Zihao and Wang, Benyou and Hu, Yan and Zou, Difan},
  booktitle    = {Proceedings of the 2025 Conference on Empirical Methods in Natural Language Processing},
  pages        = {26530--26546},
  year         = {2025},
  publisher    = {Association for Computational Linguistics},
  doi          = {10.18653/v1/2025.emnlp-main.1348}
}

@misc{dahal2026gone,
  title        = {{GONE}: Structural Knowledge Unlearning via Neighborhood-Expanded Distribution Shaping},
  author       = {Dahal, Chahana and Balasubramaniam, Ashutosh and Xiong, Zuobin},
  year         = {2026},
  eprint       = {2603.12275},
  archivePrefix = {arXiv},
  primaryClass = {cs.CL}
}

@misc{mabrok2026latentmanifolds,
  title        = {Latent Semantic Manifolds in Large Language Models},
  author       = {Mabrok, Mohamed A.},
  year         = {2026},
  eprint       = {2603.22301},
  archivePrefix = {arXiv},
  primaryClass = {cs.AI}
}

@misc{pawelczyk2023incontextunlearning,
  title        = {In-Context Unlearning: Language Models as Few Shot Unlearners},
  author       = {Pawelczyk, Martin and Neel, Seth and Lakkaraju, Himabindu},
  year         = {2023},
  eprint       = {2310.07579},
  archivePrefix = {arXiv},
  primaryClass = {cs.LG}
}

@inproceedings{sennrich2016backtranslation,
  title     = {Improving Neural Machine Translation Models with Monolingual Data},
  author    = {Sennrich, Rico and Haddow, Barry and Birch, Alexandra},
  booktitle = {Proceedings of the 54th Annual Meeting of the Association for Computational Linguistics},
  pages     = {86--96},
  year      = {2016},
  publisher = {Association for Computational Linguistics}
}

@misc{zou2023universal,
  title         = {Universal and Transferable Adversarial Attacks on Aligned Language Models},
  author        = {Zou, Andy and Wang, Zifan and Carlini, Nicholas and Nasr, Milad and Kolter, J. Zico and Fredrikson, Matt},
  year          = {2023},
  eprint        = {2307.15043},
  archivePrefix = {arXiv},
  primaryClass  = {cs.CL}
}

@article{lee1999nmf,
  title   = {Learning the Parts of Objects by Non-Negative Matrix Factorization},
  author  = {Lee, Daniel D. and Seung, H. Sebastian},
  journal = {Nature},
  volume  = {401},
  pages   = {788--791},
  year    = {1999}
}

@book{jolliffe2002pca,
  title     = {Principal Component Analysis},
  author    = {Jolliffe, Ian T.},
  publisher = {Springer},
  year      = {2002}
}

@misc{bhalla2026sparsemanifolds,
  title        = {Do Sparse Autoencoders Capture Concept Manifolds?},
  author       = {Bhalla, Usha and Fel, Thomas and Rager, Can and Feucht, Sheridan and Haklay, Tal and Wurgaft, Daniel and Boppana, Siddharth and Kowal, Matthew and Shyam, Vasudev and Merullo, Jack and Geiger, Atticus and Lubana, Ekdeep Singh},
  year         = {2026},
  eprint       = {2604.28119},
  archivePrefix = {arXiv},
  url          = {https://arxiv.org/abs/2604.28119}
}

@misc{li2024wmdp,
  title         = {The {WMDP} Benchmark: Measuring and Reducing Malicious Use with Unlearning},
  author        = {Li, Nathaniel and Pan, Alexander and Gopal, Anjali and others},
  year          = {2024},
  eprint        = {2403.03218},
  archivePrefix = {arXiv},
  primaryClass  = {cs.CL}
}

@misc{zhang2024npo,
  title         = {Negative Preference Optimization: From Catastrophic Collapse to Effective Unlearning},
  author        = {Zhang, Ruiqi and Lin, Licong and Bai, Yu and Mei, Song},
  year          = {2024},
  eprint        = {2404.05868},
  archivePrefix = {arXiv},
  primaryClass  = {cs.LG}
}

@inproceedings{zhou2025neurosurgeon,
  title     = {{LLM Neurosurgeon}: Targeted Knowledge Removal in {LLM}s Using Sparse Autoencoders},
  author    = {Zhou, Dylan and Patil, Kunal and Sun, Yifan and Lakshmanan, Karthik and Rajamanoharan, Senthooran and Conmy, Arthur},
  booktitle = {ICLR 2025 Workshop on Building Trust in Language Models and Applications},
  year      = {2025},
  note      = {\url{https://openreview.net/forum?id=aeQeXlG2Pw}}
}

@article{oord2018infonce,
  title   = {Representation Learning with Contrastive Predictive Coding},
  author  = {van den Oord, A{\"a}ron and Li, Yazhe and Vinyals, Oriol},
  journal = {arXiv preprint arXiv:1807.03748},
  year    = {2018}
}

@inproceedings{khosla2020supcon,
  title     = {Supervised Contrastive Learning},
  author    = {Khosla, Prannay and Teterwak, Piotr and Wang, Chen and
               Sarna, Aaron and Tian, Yonglong and Isola, Phillip and
               Maschinot, Aaron and Liu, Ce and Krishnan, Dilip},
  booktitle = {Advances in Neural Information Processing Systems},
  volume    = {33},
  year      = {2020}
}

@inproceedings{chen2020simclr,
  title     = {A Simple Framework for Contrastive Learning of Visual
               Representations},
  author    = {Chen, Ting and Kornblith, Simon and Norouzi, Mohammad and
               Hinton, Geoffrey},
  booktitle = {Proceedings of the 37th International Conference on
               Machine Learning},
  year      = {2020}
}

@article{mcinnes2018umap,
  title   = {UMAP: Uniform Manifold Approximation and Projection for
             Dimension Reduction},
  author  = {McInnes, Leland and Healy, John and Melville, James},
  journal = {arXiv preprint arXiv:1802.03426},
  year    = {2018}
}

@inproceedings{akiba2019optuna,
  title     = {Optuna: A Next-generation Hyperparameter Optimization
               Framework},
  author    = {Akiba, Takuya and Sano, Shotaro and Yanase, Toshihiko and
               Ohta, Takeru and Koyama, Masanori},
  booktitle = {Proceedings of the 25th ACM SIGKDD International
               Conference on Knowledge Discovery and Data Mining},
  year      = {2019}
}

@inproceedings{lin2004rouge,
  title     = {{ROUGE}: A Package for Automatic Evaluation of Summaries},
  author    = {Lin, Chin-Yew},
  booktitle = {Text Summarization Branches Out},
  pages     = {74--81},
  year      = {2004},
  publisher = {Association for Computational Linguistics}
}

@inproceedings{carlini2021extracting,
  title     = {Extracting Training Data from Large Language Models},
  author    = {Carlini, Nicholas and Tram{\`e}r, Florian
               and Wallace, Eric and Jagielski, Matthew
               and Herbert-Voss, Ariel and Lee, Katherine
               and Roberts, Adam and Brown, Tom and Song, Dawn
               and Erlingsson, {\'U}lfar and Oprea, Alina
               and Raffel, Colin},
  booktitle = {30th USENIX Security Symposium},
  pages     = {2633--2650},
  year      = {2021},
  publisher = {USENIX Association}
}

@misc{eu2016gdpr,
  author       = {{European Parliament and Council of the European Union}},
  title        = {Regulation (EU) 2016/679 of the European Parliament and of the Council (General Data Protection Regulation)},
  year         = {2016},
  note         = {Article 17: Right to erasure (`right to be forgotten')},
  howpublished = {Official Journal of the European Union, L 119},
  url          = {https://eur-lex.europa.eu/eli/reg/2016/679/oj}
}
\bibliographystyle{iclr2026_conference}

\appendix

\section{Energy Bias of Reconstruction-Based Extraction}
\label{App:intro_th}

This appendix gives the formal statement behind
Section~\ref{sec:energy-bias}: reconstruction-only objectives reward an
extractor in proportion to the \emph{energy} of what it captures, so a
low-energy forget concept is systematically deprioritized relative to
high-energy background structure. The result is a motivation, not a
theorem about trained SAEs: it isolates the incentive created by the
reconstruction loss under an idealized concept decomposition, in the
spirit of the additive concept-manifold view of
\citet{bhalla2026sparsemanifolds}. We then describe in full the PCA
experiment that measures the corresponding energy gap on real
activations (Figure~\ref{fig:energy-bias}).

\subsection{Setting}

Let $h\in\mathbb{R}^d$ be a random residual-stream activation, and
suppose it decomposes into concept-manifold components
\begin{equation}
  h \;=\; \sum_{j=1}^{J} m_j,
  \qquad m_j \in \mathcal{M}_j,
  \label{eq:app-energy-decomp}
\end{equation}
where $\mathcal{M}_j$ denotes the $j$-th concept manifold. Define the
energy of concept $j$ as
\begin{equation}
  E_j \;=\; \mathbb{E}\bigl[\lVert m_j\rVert_2^2\bigr] \;>\; 0.
  \label{eq:app-energy-def}
\end{equation}

Let $\Theta$ denote an extractor (for instance an SAE), producing a
reconstruction $\hat h^{\Theta}$ of $h$, with reconstruction loss
$\mathcal{L}_{\mathrm{rec}}(\Theta)
 = \mathbb{E}[\lVert h - \hat h^{\Theta}\rVert_2^2]$.

\begin{assumption}[Component orthogonality]
\label{ass:app-components}
The concept components are centered and mutually uncorrelated:
\[
  \mathbb{E}[m_j] = 0
  \quad\text{for all } j,
  \qquad
  \mathbb{E}\bigl[\langle m_j, m_k\rangle\bigr] = 0
  \quad\text{for } j\neq k.
\]
\end{assumption}

Under Assumption~\ref{ass:app-components}, the total activation energy
splits exactly across concepts,
$\mathbb{E}[\lVert h\rVert_2^2] = \sum_{j=1}^{J} E_j$, which is what
makes $E_j$ interpretable as concept $j$'s share of the signal.

\begin{assumption}[Concept-aligned reconstruction with orthogonal
errors]
\label{ass:app-errors}
The reconstruction admits a concept-aligned decomposition
$\hat h^{\Theta} = \sum_{j=1}^{J} \hat m_j^{\Theta}$ (any decoder bias
is absorbed into the estimates), and the per-concept reconstruction
errors are mutually uncorrelated:
\[
  \mathbb{E}\bigl[\langle m_j - \hat m_j^{\Theta},\;
                   m_k - \hat m_k^{\Theta}\rangle\bigr] = 0
  \quad\text{for } j \neq k.
\]
\end{assumption}

Assumption~\ref{ass:app-errors} holds, for example, when the concept
manifolds occupy mutually orthogonal subspaces and each estimate
$\hat m_j^{\Theta}$ and its error lie in the corresponding subspace, or
when the errors are independent across concepts. It is an idealization:
real concept components overlap, and a trained SAE does not come with a
canonical per-concept decomposition of its reconstruction. The
assumption serves to isolate the incentive structure of the loss; the
experiment of Section~\ref{sec:app-energy-experiment} provides the
empirical counterpart on real activations.

For each concept, define the \emph{captured energy fraction}
\begin{equation}
  C_j(\Theta)
  \;=\;
  1 - \frac{\mathbb{E}\bigl[\lVert m_j - \hat m_j^{\Theta}
             \rVert_2^2\bigr]}{E_j}
  \;\le\; 1,
  \label{eq:app-capture-def}
\end{equation}
with $C_j(\Theta)\ge 0$ whenever the estimate is at least as accurate
as the trivial estimate $\hat m_j^{\Theta}=0$.

\subsection{Energy-weighted decomposition of the reconstruction loss}

\begin{proposition}[Energy bias of reconstruction objectives]
\label{prop:app-energy-bias}
Under Assumptions
\ref{ass:app-components}--\ref{ass:app-errors},
\begin{equation}
  \mathcal{L}_{\mathrm{rec}}(\Theta)
  \;=\;
  \sum_{j=1}^{J} E_j\bigl(1 - C_j(\Theta)\bigr).
  \label{eq:app-loss-decomposition}
\end{equation}
Equivalently, minimizing the reconstruction loss is equivalent to
maximizing the energy-weighted capture
\begin{equation}
  G_{\mathrm{rec}}(\Theta)
  \;=\;
  \sum_{j=1}^{J} E_j\, C_j(\Theta).
  \label{eq:app-gain-def}
\end{equation}
\end{proposition}

\begin{proof}
Write $h - \hat h^{\Theta} = \sum_{j=1}^{J}(m_j - \hat m_j^{\Theta})$
using Eq.~\eqref{eq:app-energy-decomp} and
Assumption~\ref{ass:app-errors}. Expanding the squared norm,
\[
  \mathcal{L}_{\mathrm{rec}}(\Theta)
  =
  \sum_{j=1}^{J}
  \mathbb{E}\bigl[\lVert m_j - \hat m_j^{\Theta}\rVert_2^2\bigr]
  +
  \sum_{j\neq k}
  \mathbb{E}\bigl[\langle m_j - \hat m_j^{\Theta},\,
                   m_k - \hat m_k^{\Theta}\rangle\bigr].
\]
The cross terms vanish by Assumption~\ref{ass:app-errors}, and each
diagonal term equals $E_j(1-C_j(\Theta))$ by
Eq.~\eqref{eq:app-capture-def}, which gives
Eq.~\eqref{eq:app-loss-decomposition}. Since
$\mathcal{L}_{\mathrm{rec}}(\Theta)
 = \sum_j E_j - G_{\mathrm{rec}}(\Theta)$ and $\sum_j E_j$ does not
depend on $\Theta$, minimizing the former maximizes the latter.
\end{proof}

\begin{corollary}[Marginal capacity is allocated by energy]
\label{cor:app-marginal}
Consider a marginal change of the extractor that improves the capture
fractions by $(\Delta C_1,\ldots,\Delta C_J)$. The induced
reconstruction gain is
\[
  \Delta G_{\mathrm{rec}}
  \;=\;
  \sum_{j=1}^{J} E_j\,\Delta C_j .
\]
In particular, when a unit of capacity can improve either a background
concept $\mathcal{M}_b$ by $\Delta C_b$ or a forget concept
$\mathcal{M}_f$ by $\Delta C_f$, the reconstruction objective prefers
the background concept whenever
\[
  E_b\,\Delta C_b \;>\; E_f\,\Delta C_f,
\]
and, for comparable capture improvements
$\Delta C_b \approx \Delta C_f$, the preference is decided by the
energies alone.
\end{corollary}

\begin{remark}
\label{rem:app-energy-scope}
Proposition~\ref{prop:app-energy-bias} does not say that the forget
concept is unimportant to the model, only that it contributes little to
the reconstruction objective compared with broad background structure.
PCA is the extreme case: for linear orthogonal projections, maximizing
captured energy is exactly the PCA criterion, so high-variance
background directions are selected by construction. NMF and
reconstruction-only SAEs are softer versions of the same incentive.
This is the mechanism addressed by the contrastive objective of
\textsc{SCALPEL}: Proposition~\ref{prop:contrastive-selectivity}
(proved in Appendix~\ref{app:sparse_contrastive}) shows that adding
contrastive learning to reconstruction and sparsity favors sparse
representations with stronger within-target alignment and lower
target--background overlap.
\end{remark}

\subsection{Measuring the energy gap on real activations}
\label{sec:app-energy-experiment}

We now describe in full the experiment summarized in
Section~\ref{sec:energy-bias} and Figure~\ref{fig:energy-bias}, which
measures the energy gap between background structure and
target-specific structure on residual-stream activations extracted at
layer $-4$ of Qwen.

\paragraph{Background PCA.}
Let $H_{\mathrm{bg}} = H_r \cup H_w$ be the union of retain and Wiki
activations and $H_f$ the forget activations. We estimate the
background mean and covariance using only the background data,
\[
  \overline h_{\mathrm{bg}}
  = \frac{1}{|H_{\mathrm{bg}}|}\sum_{h\in H_{\mathrm{bg}}} h,
  \qquad
  \Sigma_{\mathrm{bg}}
  = \frac{1}{|H_{\mathrm{bg}}|}\sum_{h\in H_{\mathrm{bg}}}
    \bigl(h-\overline h_{\mathrm{bg}}\bigr)
    \bigl(h-\overline h_{\mathrm{bg}}\bigr)^{\!\top}.
\]
Let $\lambda_1^{\mathrm{bg}} \ge \lambda_2^{\mathrm{bg}} \ge
\lambda_3^{\mathrm{bg}} \ge \cdots$ be the eigenvalues of
$\Sigma_{\mathrm{bg}}$ and
$U_{\mathrm{bg},3}\in\mathbb{R}^{d\times 3}$ its first three
eigenvectors. The left panel of Figure~\ref{fig:energy-bias} shows the
background activations in these coordinates,
$y_{\mathrm{bg}} = U_{\mathrm{bg},3}^{\top}
 (h - \overline h_{\mathrm{bg}})$ for $h\in H_{\mathrm{bg}}$.

\paragraph{Background-orthogonal forget residuals.}
We remove the dominant background component from every forget
activation,
\[
  \widetilde h_f
  = \bigl(I - U_{\mathrm{bg},3} U_{\mathrm{bg},3}^{\top}\bigr)
    \bigl(h_f - \overline h_{\mathrm{bg}}\bigr),
  \qquad h_f \in H_f .
\]
Projecting these residuals back onto $U_{\mathrm{bg},3}$ gives zero
coordinates by construction, so visualizing the remaining forget
structure requires a second basis: we fit a second PCA directly on
$\widetilde H_f = \{\widetilde h_f : h_f\in H_f\}$. With
$\overline{\widetilde h}_f$ the mean of the residuals, define
\[
  \Sigma_f^{\perp\mathrm{bg}}
  = \frac{1}{|\widetilde H_f|}\sum_{\widetilde h_f\in\widetilde H_f}
    \bigl(\widetilde h_f - \overline{\widetilde h}_f\bigr)
    \bigl(\widetilde h_f - \overline{\widetilde h}_f\bigr)^{\!\top},
\]
with eigenvalues $\lambda_i^{f\perp\mathrm{bg}}$. The right panel of
Figure~\ref{fig:energy-bias} shows the residual forget activations in
the top-3 eigenvectors of $\Sigma_f^{\perp\mathrm{bg}}$. The two
panels therefore use two different PCA bases and must be compared
through their scale and explained energy, not their absolute
coordinates.

\paragraph{Energy comparison.}
We measure the scale of each displayed manifold by the sum of its first
three PCA eigenvalues,
\[
  \widehat E_{\mathrm{bg}}^{(3)}
  = \sum_{i=1}^{3}\lambda_i^{\mathrm{bg}},
  \qquad
  \widehat E_f^{\perp\mathrm{bg},(3)}
  = \sum_{i=1}^{3}\lambda_i^{f\perp\mathrm{bg}},
\]
and obtain
\[
  \widehat E_{\mathrm{bg}}^{(3)} = 3{,}787{,}695,
  \qquad
  \widehat E_f^{\perp\mathrm{bg},(3)} = 11{,}517,
  \qquad
  \frac{\widehat E_{\mathrm{bg}}^{(3)}}
       {\widehat E_f^{\perp\mathrm{bg},(3)}}
  \approx 329.
\]

\paragraph{What the experiment does and does not show.}
The experiment does not claim that the complete forget manifold carries
exactly $329\times$ less energy: it measures the scale of the forget
structure that \emph{remains} after removing the three dominant
background directions, which is the component a targeted intervention
must isolate. The size of the gap nevertheless supports the energy-bias
argument of Proposition~\ref{prop:app-energy-bias}: the
target-specific component is orders of magnitude smaller than the
background structure that a reconstruction-based extractor is primarily
rewarded for capturing.

\section{Sparse Contrastive Learning Promotes Target-Selective Features}
\label{app:sparse_contrastive}

We give a geometric account of why the combination of reconstruction,
sparsity, and contrastive learning used by \textsc{SCALPEL} favors
target-selective sparse features. The result characterizes the geometry
preferred by the objective when such a representation is compatible with
reconstruction; it is not a global convergence result for SAE training.

\subsection{Setting}

For a standardized residual activation $h_i$, the SAE produces
non-negative activations $z(h_i)\in\mathbb{R}_+^Q$ and reconstructs
\[
\widehat h_i
=
b_D+\sum_{q=1}^Q z_q(h_i)d_q.
\]
The contrastive loss operates on the norm-aware code
\[
\widetilde z_q(h_i)
=
z_q(h_i)\lVert d_q\rVert_2,
\qquad
\widetilde z_i
=
\bigl(
\widetilde z_1(h_i),\ldots,\widetilde z_Q(h_i)
\bigr).
\]
For every non-zero code, define
\[
\rho_i=\lVert\widetilde z_i\rVert_2,
\qquad
u_i=\frac{\widetilde z_i}{\rho_i}.
\]
Thus
\[
u_i\in\mathbb{R}_+^Q,
\qquad
\lVert u_i\rVert_2=1,
\]
and the cosine similarity used by the contrastive objective is simply
\[
\operatorname{sim}(i,j)
=
\langle u_i,u_j\rangle.
\]

We assume that decoder scales are controlled:
\[
0<d_{\min}
\leq
\lVert d_q\rVert_2
\leq
d_{\max}<\infty.
\tag{A1}
\]
This removes the reciprocal rescaling ambiguity between encoder
activations and decoder directions.

For anchor $i$, let $P(i)$ and $N(i)$ be its positive and negative sets,
with $p_i=|P(i)|$ and $n_i=|N(i)|$. In \textsc{SCALPEL}, positives
belong to the same target concept and negatives are drawn from the
background. The multi-positive InfoNCE loss is
\[
\ell_i
=
-\log
\frac{
\sum_{j\in P(i)}
e^{\langle u_i,u_j\rangle/\tau}
}{
\sum_{j\in P(i)}
e^{\langle u_i,u_j\rangle/\tau}
+
\sum_{j\in N(i)}
e^{\langle u_i,u_j\rangle/\tau}
}.
\]

\subsection{Reconstruction prevents the trivial representation}

The contrastive direction $u_i$ is undefined when
$\widetilde z_i=0$. Reconstruction prevents the all-zero solution from
being useful.

\begin{lemma}[Non-triviality from reconstruction]
\label{lem:nontrivial}
If $z(h)=0$ almost surely, the best possible reconstruction is the
constant decoder bias
\[
b_D^\star=\mathbb{E}[h],
\]
with loss
\[
L_{\mathrm{const}}
=
\mathbb{E}
\left[
\lVert h-\mathbb{E}[h]\rVert_2^2
\right].
\]
Consequently, any SAE satisfying
\[
L_{\mathrm{rec}}<L_{\mathrm{const}}
\]
must use non-zero latent codes on a set of positive probability.
\end{lemma}

\begin{proof}
If $z(h)=0$, then $\widehat h=b_D$ for every input. The minimizer of
$\mathbb{E}\lVert h-b_D\rVert_2^2$ is $b_D=\mathbb{E}[h]$, giving the
stated loss. Any reconstruction below this value therefore requires
non-zero latent activations.
\end{proof}

Reconstruction therefore supplies a non-trivial representation, while
the remaining terms determine how this representation is organized.

\subsection{Sparsity concentrates the norm-aware code}

For a normalized code $u$, let
$u_{(1)}\geq\cdots\geq u_{(Q)}$ denote its coordinates in decreasing
order.

\begin{lemma}[Sparse energy concentration]
\label{lem:sparse_energy}
For every $K\geq1$,
\[
\sum_{q>K}u_{(q)}^2
\leq
\frac{\lVert u\rVert_1^2}{K},
\]
and hence
\[
\sum_{q=1}^{K}u_{(q)}^2
\geq
1-\frac{\lVert u\rVert_1^2}{K}.
\]
Moreover,
\[
\lVert u\rVert_1\geq\lVert u\rVert_2=1,
\]
with equality if and only if $u$ is one-sparse.
\end{lemma}

\begin{proof}
Since the coordinates are sorted,
\[
u_{(K+1)}
\leq
\frac{\lVert u\rVert_1}{K}.
\]
Therefore
\[
\sum_{q>K}u_{(q)}^2
\leq
u_{(K+1)}
\sum_{q>K}u_{(q)}
\leq
\frac{\lVert u\rVert_1^2}{K}.
\]
The second result follows from $\lVert u\rVert_2^2=1$.
Finally, $\lVert u\rVert_1\geq\lVert u\rVert_2$, with equality for a
non-negative vector only when at most one coordinate is non-zero.
\end{proof}

Under Assumption~(A1), sparsifying $z$ also sparsifies
$\widetilde z$, since multiplication by bounded positive decoder norms
does not change its support. Thus the $\ell_1$ term favors concentrating
the representation on a small number of coordinates, while
reconstruction prevents complete collapse.

\subsection{Contrastive learning aligns targets and separates background}

We now characterize the geometry preferred by the contrastive term.

\begin{proposition}[Optimal contrastive geometry]
\label{prop:contrastive_geometry}
For every valid anchor $i$,
\[
\ell_i
\geq
\log\left(
1+\frac{n_i}{p_i}e^{-1/\tau}
\right).
\]
Equality holds if and only if
\[
\langle u_i,u_j\rangle=1
\quad\forall j\in P(i),
\]
and
\[
\langle u_i,u_j\rangle=0
\quad\forall j\in N(i).
\]
Therefore, whenever this geometry is compatible with reconstruction,
the contrastive objective favors identical normalized codes within a
target and disjoint supports between that target and its background
negatives.
\end{proposition}

\begin{proof}
Because $u_i,u_j$ are non-negative unit vectors,
\[
0\leq\langle u_i,u_j\rangle\leq1.
\]
Define
\[
A_i^+
=
\sum_{j\in P(i)}
e^{\langle u_i,u_j\rangle/\tau},
\qquad
A_i^-
=
\sum_{j\in N(i)}
e^{\langle u_i,u_j\rangle/\tau}.
\]
Then
\[
A_i^+\leq p_i e^{1/\tau},
\qquad
A_i^-\geq n_i.
\]
Since
\[
\ell_i
=
\log\left(1+\frac{A_i^-}{A_i^+}\right),
\]
we obtain
\[
\ell_i
\geq
\log\left(
1+\frac{n_i}{p_i}e^{-1/\tau}
\right).
\]

Equality requires simultaneously
\[
\langle u_i,u_j\rangle=1
\quad (j\in P(i))
\]
and
\[
\langle u_i,u_j\rangle=0
\quad (j\in N(i)).
\]
For unit vectors, the first condition implies $u_i=u_j$. Because all
coordinates are non-negative, the second condition implies
\[
\operatorname{supp}(u_i)
\cap
\operatorname{supp}(u_j)
=
\varnothing.
\]
Thus same-target examples share a normalized representation, whereas
target and background negatives use disjoint coordinates.
\end{proof}

Combined with Lemma~\ref{lem:sparse_energy}, this gives the preferred
geometry of the joint objective: reconstruction keeps the representation
non-trivial, sparsity concentrates it onto few coordinates, and the
contrastive term encourages those coordinates to be shared by examples
of the same target while remaining inactive on background examples.

Importantly, the proposition only separates pairs defined as negatives.
If two forget targets are not contrasted against each other, they may
share coordinates. This is consistent with the shared features observed
empirically in \textsc{SCALPEL}.

\subsection{Near-optimal contrastive loss implies near-optimal geometry}

The previous result describes the exact optimum. We can also quantify how
approaching this optimum forces the representations toward the same
geometry.

Let
\[
R_i=\frac{A_i^-}{A_i^+},
\qquad
R_i^\star
=
\frac{n_i}{p_i}e^{-1/\tau}.
\]
Suppose
\[
R_i
\leq
(1+\delta)R_i^\star,
\qquad
\delta\geq0.
\tag{A2}
\]

\begin{proposition}[Approximate contrastive separation]
\label{prop:approx_contrastive}
Under Eq.~(A2),
\[
\frac{1}{n_i}
\sum_{j\in N(i)}
\langle u_i,u_j\rangle
\leq
\tau\delta,
\]
and
\[
\frac{1}{p_i}
\sum_{j\in P(i)}
\left(
1-\langle u_i,u_j\rangle
\right)
\leq
\frac{\delta}
{(1+\delta)(1-e^{-1/\tau})}.
\]
Consequently,
\[
\frac{1}{p_i}
\sum_{j\in P(i)}
\lVert u_i-u_j\rVert_2^2
\leq
\frac{2\delta}
{(1+\delta)(1-e^{-1/\tau})}.
\]
Hence, as the contrastive loss approaches its optimum, positive codes
collapse toward one another while their average similarity with
background negatives approaches zero.
\end{proposition}

\begin{proof}
Observe that
\[
\frac{R_i}{R_i^\star}
=
\underbrace{
\frac{A_i^-}{n_i}
}_{\geq1}
\underbrace{
\frac{p_i e^{1/\tau}}{A_i^+}
}_{\geq1}.
\]
If their product is at most $1+\delta$, each factor is itself at most
$1+\delta$.

For the negatives,
\[
\frac{1}{n_i}
\sum_{j\in N(i)}
e^{\langle u_i,u_j\rangle/\tau}
\leq
1+\delta.
\]
Using $e^x\geq1+x$ gives
\[
\frac{1}{n_i}
\sum_{j\in N(i)}
\langle u_i,u_j\rangle
\leq
\tau\delta.
\]

For the positives,
\[
\frac{1}{p_i}
\sum_{j\in P(i)}
e^{-(1-\langle u_i,u_j\rangle)/\tau}
\geq
\frac{1}{1+\delta}.
\]
For $x\in[0,1]$, concavity of $1-e^{-x/\tau}$ gives
\[
1-e^{-x/\tau}
\geq
(1-e^{-1/\tau})x.
\]
Therefore,
\[
(1-e^{-1/\tau})
\frac{1}{p_i}
\sum_{j\in P(i)}
\left(1-\langle u_i,u_j\rangle\right)
\leq
1-\frac{1}{1+\delta}
=
\frac{\delta}{1+\delta},
\]
which proves the second inequality. Finally,
\[
\lVert u_i-u_j\rVert_2^2
=
2\left(1-\langle u_i,u_j\rangle\right),
\]
giving the last result.
\end{proof}

Since
\[
\ell_i=\log(1+R_i),
\]
condition~(A2) is equivalently a bound on the excess contrastive loss:
$\delta\rightarrow0$ whenever
$\ell_i\rightarrow\ell_i^\star$.

\subsection{Connection to the selectivity score}

Recall the \textsc{SCALPEL} score
\[
s_q(G_a)
=
\frac{\mu_q(G_a)}
{\mu_q(\mathrm{bg})+\epsilon},
\]
where
\[
\mu_q(G_a)
=
\mathbb{E}_{i\in G_a}
[\widetilde z_q(h_i)],
\qquad
\mu_q(\mathrm{bg})
=
\mathbb{E}_{j\in\mathrm{bg}}
[\widetilde z_q(h_j)].
\]

At the ideal sparse contrastive geometry, suppose target $G_a$ uses
coordinate $q$. Reconstruction ensures a non-zero code scale, sparsity
concentrates the target representation on $q$, and
Proposition~\ref{prop:contrastive_geometry} gives
\[
\widetilde z_q(h_i)>0
\quad\text{for }i\in G_a,
\qquad
\widetilde z_q(h_j)=0
\quad\text{for }j\in\mathrm{bg}.
\]
Hence
\[
\mu_q(G_a)>0,
\qquad
\mu_q(\mathrm{bg})=0,
\]
and therefore
\[
s_q(G_a)
=
\frac{\mu_q(G_a)}{\epsilon}.
\]

Thus the geometry favored by the joint objective is exactly the geometry
rewarded by the selectivity criterion: the sparsity term concentrates
target information onto few coordinates, while the contrastive term
increases agreement within the target and suppresses overlap of those
coordinates with background examples. Proposition~\ref{prop:approx_contrastive}
further shows that this behavior emerges continuously as the contrastive
loss approaches its optimum.

\section{Selectivity and Intervention Safety}
\label{App:score_unlearing}

This appendix proves Proposition~\ref{prop:retain-damage} of the main
text: the selectivity score used by \textsc{SCALPEL} controls the
expected magnitude of the perturbation applied to background activations.
Together with the sparse contrastive analysis of
Appendix~\ref{app:sparse_contrastive}, it completes the
training-to-intervention argument: contrastive learning favors sparse
representations with stronger within-target alignment and lower
target--background overlap, while selecting high-score features limits
their expected footprint on retain and Wiki representations under
intervention.

\subsection{Setting}

We use the notation of Sections~\ref{sec:contrastive-sae}
and~\ref{sec:feature-selection}. For a standardized activation
$\overline h$, the encoder produces a non-negative sparse code
$z(\overline h)\in\mathbb{R}_+^Q$, so that
\begin{equation}
  z_q(\overline h) \;\ge\; 0
  \qquad\text{for all } q,\ \overline h,
  \label{eq:app-nonneg}
\end{equation}
and $d_q\in\mathbb{R}^d$ is the decoder direction of feature $q$. For
an author-level target $G_a$, with
$\overline{\mathcal H}_{\mathrm{bg}}
 = \overline{\mathcal H}_r\cup\overline{\mathcal H}_w$
the standardized background activations, recall the mean norm-aware
activations and the selectivity score
(Eq.~\eqref{eq:score} of the main text),
\[
  \mu_q(G_a)
  = \mathbb{E}_{\overline h\sim\overline{\mathcal H}_{G_a}}
    \bigl[z_q(\overline h)\,\lVert d_q\rVert_2\bigr],
  \qquad
  \mu_q(\mathrm{bg})
  = \mathbb{E}_{\overline h\sim\overline{\mathcal H}_{\mathrm{bg}}}
    \bigl[z_q(\overline h)\,\lVert d_q\rVert_2\bigr],
  \qquad
  s_q(G_a)
  = \frac{\mu_q(G_a)}{\mu_q(\mathrm{bg})+\varepsilon},
\]
with $\varepsilon>0$. \textsc{SCALPEL} selects
$S_{G_a} = \{q : s_q(G_a) \ge s_{\min}\}$ for a threshold
$s_{\min}>0$, and the intervention applies the perturbation
\begin{equation}
  \delta_{G_a}(\overline h)
  \;=\;
  \alpha \sum_{q\in S_{G_a}} z_q(\overline h)\, d_q,
  \qquad \alpha \ge 0.
  \label{eq:app-perturbation}
\end{equation}

\subsection{The perturbation bound}

\begin{proposition}[Selectivity controls background perturbation;
Proposition~\ref{prop:retain-damage} of the main text]
\label{prop:app-perturbation}
If every selected feature satisfies $s_q(G_a)\ge s_{\min}$, then
\begin{equation}
  \mathbb{E}_{\overline h\sim\overline{\mathcal H}_{\mathrm{bg}}}
  \bigl[\lVert \delta_{G_a}(\overline h)\rVert_2\bigr]
  \;\le\;
  \alpha \sum_{q\in S_{G_a}} \mu_q(\mathrm{bg})
  \;\le\;
  \frac{\alpha}{s_{\min}} \sum_{q\in S_{G_a}} \mu_q(G_a)
  \;-\;
  \alpha\,\varepsilon\,\bigl|S_{G_a}\bigr|
  \;\le\;
  \frac{\alpha}{s_{\min}} \sum_{q\in S_{G_a}} \mu_q(G_a).
  \label{eq:app-bound}
\end{equation}
\end{proposition}

\begin{proof}
By the triangle inequality and the non-negativity of the codes
(Eq.~\eqref{eq:app-nonneg}), for every $\overline h$,
\[
  \lVert \delta_{G_a}(\overline h)\rVert_2
  \;\le\;
  \alpha \sum_{q\in S_{G_a}}
  z_q(\overline h)\,\lVert d_q\rVert_2 .
\]
Taking the expectation over
$\overline h\sim\overline{\mathcal H}_{\mathrm{bg}}$ and using
linearity gives the first inequality of Eq.~\eqref{eq:app-bound}. For
the second, each selected feature satisfies
$s_q(G_a) = \mu_q(G_a)/(\mu_q(\mathrm{bg})+\varepsilon)
 \ge s_{\min}$, hence
\[
  \mu_q(\mathrm{bg})
  \;\le\;
  \frac{\mu_q(G_a)}{s_{\min}} - \varepsilon .
\]
Summing over $q\in S_{G_a}$ and multiplying by $\alpha$ yields the
second inequality; dropping the non-negative term
$\alpha\varepsilon|S_{G_a}|$ yields the third.
\end{proof}

\subsection{Discussion}

\begin{remark}[Interpretation and monotonicity]
\label{rem:app-bound-interpretation}
The quantity $\sum_{q\in S_{G_a}}\mu_q(G_a)$ is the total norm-aware
activation mass of the selected features on the target: it is the
natural scale of the intervention, since the same triangle-inequality
argument applied to target activations gives
$\mathbb{E}_{\overline h\sim\overline{\mathcal H}_{G_a}}
 [\lVert\delta_{G_a}(\overline h)\rVert_2]
 \le \alpha\sum_{q\in S_{G_a}}\mu_q(G_a)$.
Proposition~\ref{prop:app-perturbation} states that the corresponding
bound on \emph{background} data is smaller by the factor
$1/s_{\min}$: relative to the activation mass they carry on the
target, high-selectivity features can only perturb background
activations weakly. The bound is monotone in the threshold, raising
$s_{\min}$ tightens it, and the intermediate form in
Eq.~\eqref{eq:app-bound} shows that the constant $\varepsilon$ in the
score buys an additional slack of $\alpha\varepsilon$ per selected
feature.
\end{remark}

\begin{remark}[Standardized versus original residual space]
\label{rem:app-standardization}
The bound is stated in standardized activation space, where the SAE
operates (Section~\ref{sec:activation-extraction}), while the
intervention of Eq.~\eqref{eq:intervention} is applied to the residual
stream. If the decoded contribution is mapped back to the original
space through the fixed standardization, the perturbation becomes
$\alpha\sum_{q\in S_{G_a}} z_q(\overline h)\,
 (\sigma_H \odot d_q)$, and
Proposition~\ref{prop:app-perturbation} holds verbatim with
$\lVert d_q\rVert_2$ replaced by
$\lVert \sigma_H \odot d_q\rVert_2
 \le \lVert\sigma_H\rVert_\infty\,\lVert d_q\rVert_2$ in the
definition of $\mu_q$; in particular the same bound holds up to the
fixed factor $\lVert\sigma_H\rVert_\infty$. The selectivity ratio
$s_q(G_a)$, and hence the $1/s_{\min}$ gain, is unchanged by any
per-feature rescaling of $\lVert d_q\rVert_2$, since the rescaling
appears in both the numerator and the denominator of the score.
\end{remark}

\begin{remark}[What the bound does and does not guarantee]
\label{rem:app-bound-scope}
Proposition~\ref{prop:app-perturbation} is a
\emph{representation-level safety bound}, and its scope should be read
precisely. First, it controls the \emph{expected} perturbation
magnitude over background activations, not the worst case: individual
background tokens on which a selected feature fires strongly can still
be perturbed substantially. Second, it aggregates over the selected
set, so the bound grows with $|S_{G_a}|$; selecting few, highly
selective features is favored on both counts. Third, and most
importantly, it says nothing about model \emph{behavior}: relating a
perturbation at layer $\ell$ to a change in the output distribution
would require additional assumptions on the downstream layers. Fourth,
the score measures activation magnitude and decoder norm but not the
\emph{orientation} of $d_q$ relative to the local geometry of the
forget and background manifolds; two features with similar scores can
therefore produce different behavioral effects, a variability that
Section~\ref{sec:score-results} measures directly and that a
geometry-aware score could reduce at substantially higher cost
(Appendix~\ref{app:limitations_full}).
\end{remark}

\begin{remark}[Connection to the training analysis]
\label{rem:app-bound-training}
Appendix~\ref{app:sparse_contrastive} shows that the contrastive objective
favors sparse representations with stronger within-target alignment and
lower target--background overlap, thereby promoting features with higher
selectivity scores $s_q(G_a)$. Proposition~\ref{prop:app-perturbation}
completes the argument: selecting features above a threshold
$s_{\min}$ limits their expected perturbation on background
representations relative to their activation on the target.
\end{remark}

\section{UMAP : mainfold of the datasets}
\label{app:UMAP}
\begin{figure}[!ht]
    \centering
    \includegraphics[width=0.72\linewidth]{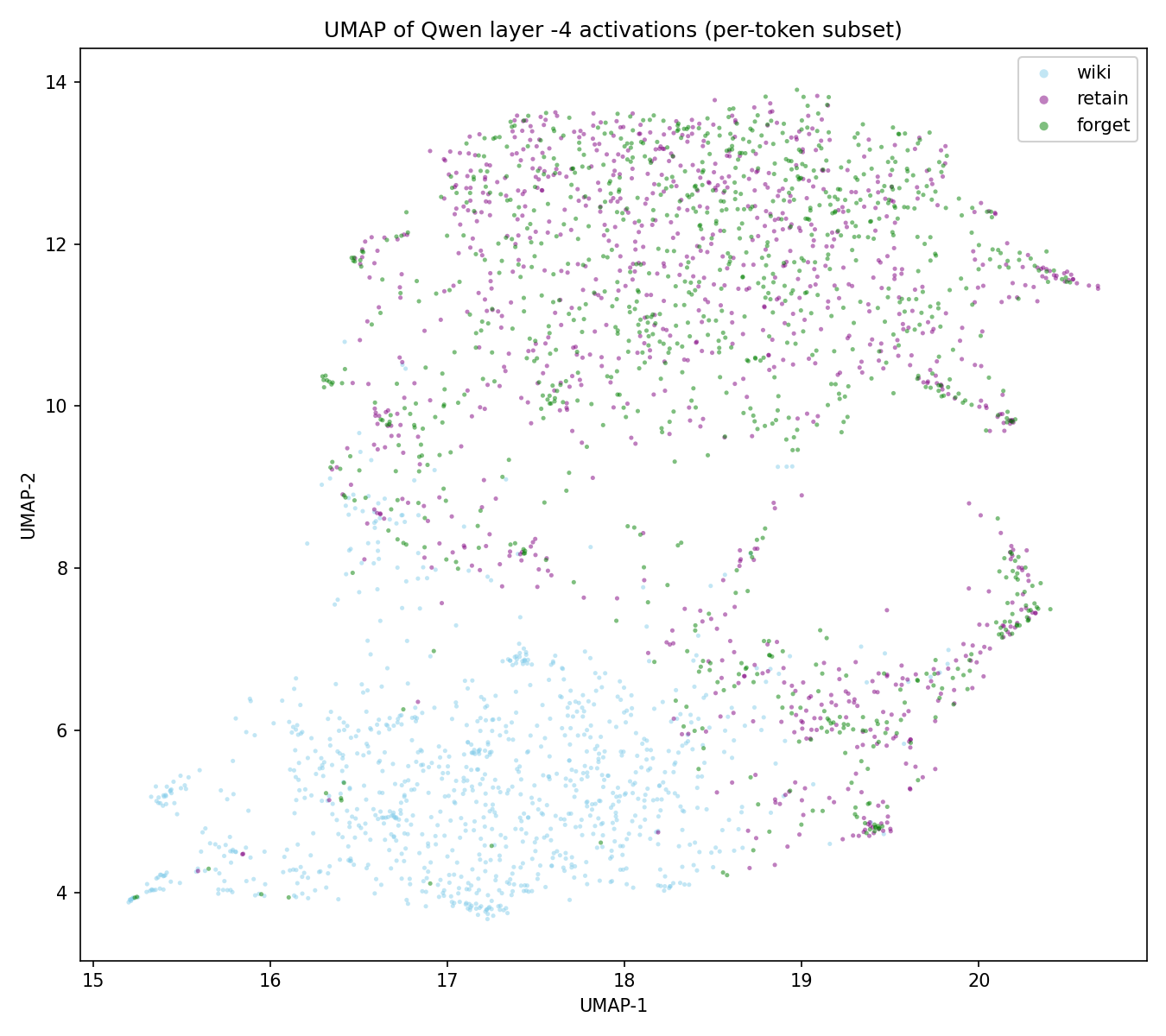}
    \caption{
    UMAP projection~\citep{mcinnes2018umap} of residual-stream activations
    from the forget, retain, and Wiki sets. Forget and retain examples
    strongly overlap because they come from the same generation process
    and differ only through the random assignment of authors to the
    splits: the full forget set is not a globally separable concept in
    activation space.
    }
    \label{fig:umap-target-granularity}
\end{figure}

\section{Hyperparameter Optimization Protocol and Search Spaces}
\label{app:hpo}
 
This appendix specifies the complete model-selection protocol. It has
two levels. \textbf{Level 1} is the hyperparameter search: one
independent Optuna study~\citep{akiba2019optuna} per
(model~$\times$~method) pair ,  $3$ models $\times$ ($5$ SAE-family
methods $+$ $2$ trainer methods) $=21$ studies ,  all run with a
5 seeds (0-4). A study never samples over methods: each method is
pinned to its own study and therefore receives its own full trial
budget. \textbf{Level 2} freezes the winning configuration of each
study and reruns the complete pipeline over five independent seeds;
its essentials are summarized in
Section~\ref{sec:app-hpo-multiseed} below, and the reported results
come exclusively from Level 2.
 
\subsection{Selection objective}
\label{sec:app-hpo-objective}
 
All $21$ studies maximize the same scalar, the
\emph{retain--forget margin}
\begin{equation}
  \mathrm{margin}
  \;=\;
  \tfrac{1}{2}\bigl(P_{\mathrm{retain}} + P_{\mathrm{holdout}}\bigr)
  \;-\;
  P_{\mathrm{forget}},
  \label{eq:app-hpo-margin}
\end{equation}
where $P_{\mathrm{forget}}$, $P_{\mathrm{retain}}$, and
$P_{\mathrm{holdout}}$ are teacher-forced ground-truth answer
probabilities on OpenUnlearning's own probe sets
(\texttt{forget05} and \texttt{holdout05}; probe batch size $32$;
parity with the OpenUnlearning data verified by a dedicated check).
Model selection never uses generation-based or test metrics, so the
selection signal is decoupled from the reported evaluation. An earlier
constrained objective (smallest $\alpha$ such that
$P_{\mathrm{forget}}$ falls below a threshold) is kept in the codebase
for diagnostics only; all reported studies use the margin of
Eq.~\eqref{eq:app-hpo-margin}.
 
\paragraph{Operating point within a trial (SAE track).}
For representation-based methods, the free knob inside a trial is the
intervention strength $\alpha$. Each trial trains the extractor,
selects features, then evaluates a coarse grid
$\alpha \in \{0, 2, 4, 8, 16, 32\}$ in a single prefix-shared forward
pass, followed by golden-section refinement (ratio $0.618$) on the
bracket around the coarse argmax, with a budget of at most $8$
refinement evaluations and a relative resolution of $0.05$. The trial
value is the best margin over all evaluated $\alpha$. This fast,
margin-maximizing search is used only for model selection; it is
distinct from the log-space bisection of
Appendix~\ref{app:alpha_search}, which traces the full trade-off
curves reported in the benchmark.
 
\paragraph{Operating point within a trial (trainer track).}
GradDiff and RMU have no $\alpha$; the free knob is the training
checkpoint. Every checkpoint (stride $1$) is scored with the same
margin, and the trial value is the maximum over
\emph{non-degenerate} checkpoints. A degeneracy guard rejects
checkpoints whose forget probability falls below $0.1\times$ the
retain-reference threshold as collapsed rather than unlearned (the
threshold is computed from a retain-only finetuned reference); if all
checkpoints are degenerate, the trial receives a penalty of $-10^{9}$.
 
\subsection{Search spaces}
\label{sec:app-hpo-spaces}
 
Table~\ref{tab:app-hpo-sae} gives the SAE-family search space and
Table~\ref{tab:app-hpo-trainers} the trainer search space; the
distributions below are exactly those recorded in the study
databases.
 
\begin{table}[h]
\centering
\caption{SAE-family search space (per trial). ``Trained SAEs'' means
all methods except NMF; ``contrastive family'' means all contrastive
variants.}
\label{tab:app-hpo-sae}
\small
\begin{tabular}{llll}
\toprule
Hyperparameter & Range / choices & Sampling & Applies to \\
\midrule
Latent dimension $Q$ & $\{64, 512, 1024, 2048, 4096, 8192\}$ & categorical & all \\
Selection threshold $s_{\min}$ & $[1.0,\,100.0]$ & uniform & all \\
Batch size & $\{64, 128, 256, 512\}$ & categorical & trained SAEs \\
Training epochs & $[5,\,40]$ & int.\ uniform & trained SAEs \\
Encoder depth & $\{1, 2\}$ & categorical & trained SAEs \\
Sparsity coefficient $\lambda$ & $[10^{-4},\,10^{-1}]$ & log-uniform & trained SAEs \\
Temperature $\tau$ & $\{0.1, 0.5\}$ & categorical & contrastive family \\
Contrastive weight $\beta$ & $[10^{-4},\,1.0]$ & log-uniform & contrastive SAE \\
Common-block fraction & $\{0.25, 0.5, 0.75\}$ & categorical & Split  \\
$\beta_{\mathrm{com}}$, $\beta_{\mathrm{diff}}$ & $[10^{-4},\,1.0]$ each & log-uniform & Split\\
NMF max iterations & $[50,\,500]$ & int.\ uniform & NMF \\
\bottomrule
\end{tabular}
\end{table}
 
\emph{Fixed (not searched), SAE family:} learning rate $10^{-3}$;
extraction layer $-4$ (fourth from last); Wiki reconstruction weight
$\lambda_{\mathrm{wiki}}=1.0$, so the training loss is
$\mathrm{MSE}(\mathrm{TOFU})
 + \lambda_{\mathrm{wiki}}\,\mathrm{MSE}(\mathrm{Wiki})
 + \lambda\,\mathcal{L}_{\mathrm{sparse}}
 + \text{contrastive terms}$;
author-level contrastive grouping
(Section~\ref{sec:target-granularity}); per-dimension z-score input
normalization computed once from the TOFU token cache and frozen
across trials; feature selection by the per-author norm-aware
forget/background ratio of Eq.~\eqref{eq:score} with top-1 fallback,
using at most $20{,}000$ tokens per side.
 \begin{table}[h]
\centering
\caption{Trainer search space (GradDiff and RMU). The names
\texttt{gamma} and \texttt{alpha} are the trainer implementation's
loss weights and are unrelated to the paper's intervention strength
$\alpha$ and norm exponent $\gamma$.}
\label{tab:app-hpo-trainers}
\small
\begin{tabular}{llll}
\toprule
Hyperparameter & Range / choices & Sampling & Method \\
\midrule
Learning rate & $[10^{-6},\,10^{-4}]$ & log-uniform & both \\
Training epochs & $\{10, 15, 20, 25, 30\}$ & categorical & both \\
Forget-loss weight (\texttt{gamma}) & $[0.1,\,5.0]$ & log-uniform & both \\
Retain-loss weight (\texttt{alpha}) & $[0.1,\,5.0]$ & log-uniform & both \\
Retain loss type & $\{\mathrm{NLL}, \mathrm{KL}\}$ & categorical & GradDiff \\
Steering coefficient & $[0.5,\,40.0]$ & log-uniform & RMU \\
\bottomrule
\end{tabular}
\end{table}
 
\emph{Fixed, trainers:} per-device train batch size $4$; RMU retain
loss \texttt{EMBED\_DIFF}; checkpoints saved per epoch (model
weights only); OpenUnlearning's generation-based evaluation disabled
during training. The epoch grid was raised from the default
$\{2,4,6,8,10\}$ to $\{10,\dots,30\}$ because the retain utility of
GradDiff and RMU is non-monotonic in training and recovers late.
 
\subsection{Search algorithm, budget, and pruning}
\label{sec:app-hpo-algo}
 
All studies use Optuna's TPE sampler with seed $0$, direction
\emph{maximize}. SAE-family studies additionally use a median pruner
($3$ startup trials): the pruning signal reported before refinement
is the best margin on the coarse $\alpha$ grid, so a pruned trial
skips the golden-section refinement but its record is still written.
Trainer studies report the same intermediate value but never prune,
since training has already happened by scoring time. The budget is
$25$ trials per SAE-family study and $20$ per trainer study; the
budget is a total target, and resubmitted jobs resume the persistent
study database until it is reached, which makes walltime-limited
cluster jobs safe. For cost control, the base model and the
${\sim}25$\,GB activation cache are loaded once per study rather than
once per trial. A built-in uniform-random searcher over the same
spaces exists as a fallback, but all reported results use
Optuna/TPE.
 
\subsection{Study outcomes and winning configurations}
\label{sec:app-hpo-outcomes}
 
Table~\ref{tab:app-hpo-outcomes} reports, for every study, the number
of completed and pruned trials (pruned trials consume budget) and the
best margin found, as of the reported snapshot; resumable studies
below their trial target are marked by their counts.
 
\begin{table}[h]
\centering
\caption{Per-study outcomes: completed trials / pruned trials / best
margin (Eq.~\eqref{eq:app-hpo-margin}).}
\label{tab:app-hpo-outcomes}
\small
\begin{tabular}{lccc}
\toprule
Method & Qwen2.5 & Llama3.2 & Gemma2 \\
\midrule
Contrastive SAE & $7\,/\,18\,/\,0.528$ & $16\,/\,9\,/\,0.532$ & $7\,/\,18\,/\,0.518$ \\
\textsc{SCALPEL}-Split & $17\,/\,8\,/\,0.526$ & $11\,/\,14\,/\,0.533$ & $9\,/\,15\,/\,0.517$ \\
Standard SAE & $15\,/\,10\,/\,0.073$ & $15\,/\,8\,/\,0.088$ & $12\,/\,13\,/\,0.313$ \\
NMF & $20\,/\,5\,/\,0.020$ & $18\,/\,7\,/\,{\approx}0$ & $11\,/\,14\,/\,0.020$ \\
GradDiff & $20\,/\,0\,/\,0.509$ & $20\,/\,0\,/\,0.517$ & $20\,/\,0\,/\,0.527$ \\
RMU & $20\,/\,0\,/\,0.531$ & $20\,/\,0\,/\,0.526$ & $20\,/\,0\,/\,0.524$ \\
\bottomrule
\end{tabular}
\end{table}
 
The selection margins already anticipate the benchmark:
reconstruction-only extraction (standard SAE, NMF) cannot reach a
useful margin on any model, while the contrastive variants match the
strong trainer baselines at selection time.
 
Each study's winning configuration is copied verbatim into the
experiment configuration by a single generator script; no
configuration is hand-edited. As an example, the winning Qwen
contrastive-SAE configuration is: latent dimension $1024$, encoder
depth $2$, $19$ epochs, batch size $256$, sparsity coefficient
$2.34\times 10^{-3}$, $\tau=0.1$, contrastive weight
$2.95\times 10^{-2}$, selection threshold $12.7$.
 
\paragraph{One documented post-hoc override.}
For GradDiff only, the winning epoch count is replaced by a single
epoch with sub-epoch checkpointing (a checkpoint every $20$ steps,
${\sim}40$ checkpoints, degeneracy floor $0.1$): at the selected
learning rate GradDiff collapses the forget set inside the first
epoch, so per-epoch checkpoints only ever contain already-collapsed
models, and sub-epoch checkpointing is required to resolve the
transition. This is the only manual deviation from the search
winners, and it is applied identically across models and seeds.
 
\subsection{From frozen configurations to the multi-seed benchmark}
\label{sec:app-hpo-multiseed}
 
With hyperparameters frozen, Level 2 runs five independent end-to-end
pipelines per model (finetuning seeds $0$--$4$; unlearning and SAE
seeds fixed to $0$), so the reported error bars cover the full
pipeline, not just the unlearning stage. For each seed: (i) the full
model and the retain-only reference are finetuned from the base
checkpoint (learning rate $5\times 10^{-5}$, $5$ base epochs with a
balanced $\{-1,0,+1\}$ epoch jitter dealt round-robin and shuffled
deterministically, so no jitter level is missed over five seeds; full
and retain share the seed's epoch count to remain comparable); (ii)
the forget side uses all $13{,}000$ augmented forget rows, and the
retain side a seed-deterministic random sample of $64{,}000$ rows
from the augmented retain pool, while probe and evaluation data are
identical across seeds; (iii) a divergence gate verifies that the
full and retain models actually separate on the forget set before
the seed's branch runs; (iv) the activation cache (layer $-4$) and
the retain-reference threshold are re-extracted from that seed's own
models; (v) every method retrains with its frozen configuration and
re-selects its operating point against that seed's threshold , 
SAE-family methods rerun the $\alpha$ search, trainers re-select the
checkpoint; and (vi) each winner receives a full OpenUnlearning TOFU
evaluation and one row in the results registry. The whole pipeline
is emitted as a dependency-ordered SLURM DAG, resumable through
per-stage sentinels. Aggregation follows
Appendix~\ref{app:alpha_search}: per-method mean $\pm$ population
standard deviation across seeds; where a cell has fewer than five
completed seeds at the reported snapshot, we state the per-cell $n$
alongside the error bars.
 
\subsection{Compute footprint}
\label{sec:app-hpo-compute}
 
Each study runs on a single H100 with $64$--$128$\,GB of RAM and $6$
CPUs; SAE-family studies take $4$--$24$\,h of walltime (resumed until
the trial target), trainer studies $12$--$24$\,h. Benchmark stages
run on smaller MIG slices where they fit (SAE training and $\alpha$
search on a \texttt{2g.20gb} slice, raised to \texttt{3g.40gb} for
Gemma-2 due to its hidden size of $2304$ and $256$k vocabulary).
Gemma-2 uses eager attention throughout, as the SDPA soft-capping
implementation produced numerical instabilities.

\section{Benchmarks for Qwen, llama and Gemma}
\subsection{Model Utility vs forget prob}
\label{app:additional_benchmarks}
\begin{figure*}[!ht]
    \centering
    \includegraphics[width=\textwidth]
    {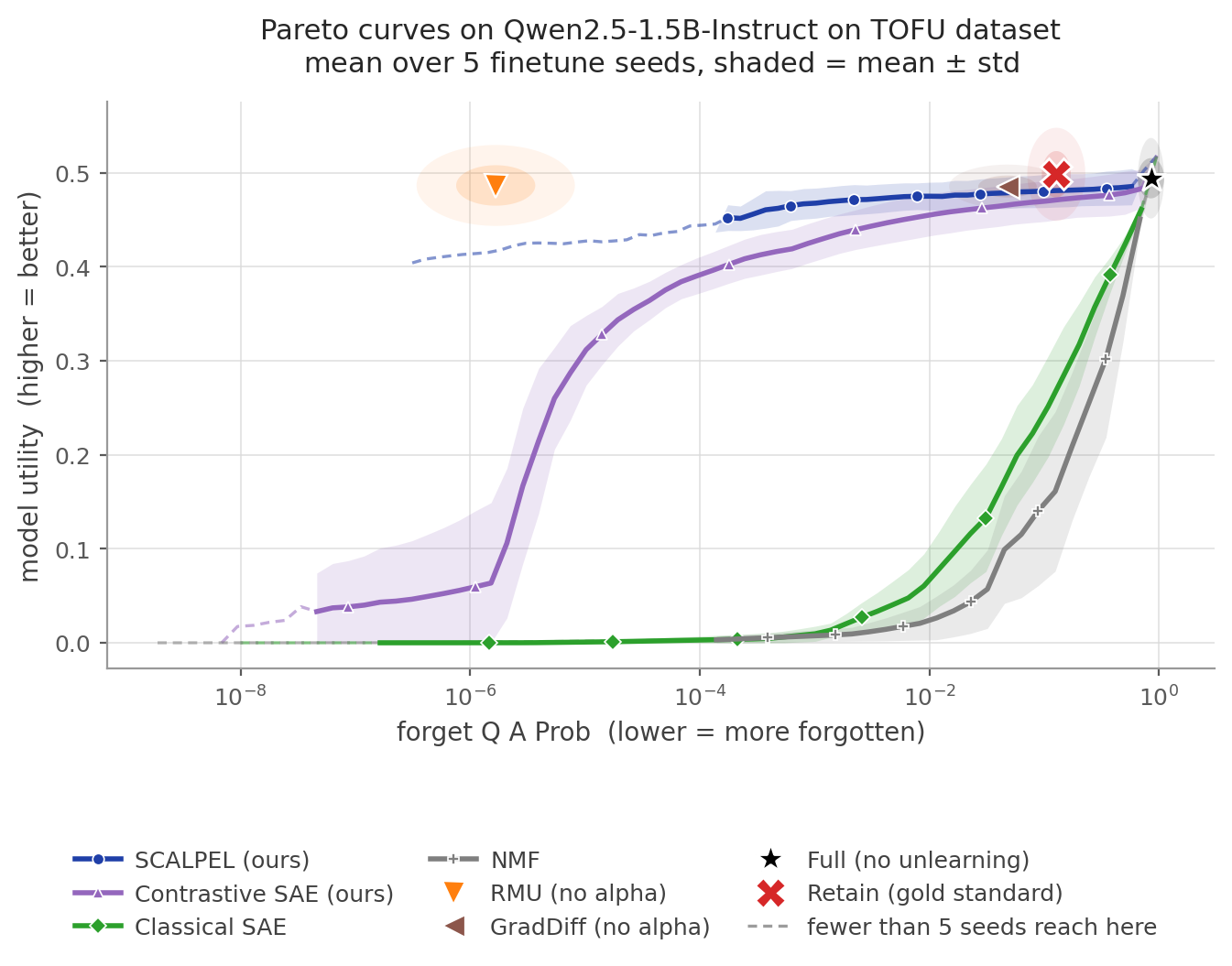}
    \caption{
    Unlearning trade-off on TOFU forget10 for
    \texttt{Qwen2.5-1.5B-Instruct}, after method-specific
    hyperparameter optimization and aggregation over five independent
    end-to-end seeds. Each curve varies the intervention strength:
    lower forget probability indicates stronger forgetting, while higher
    model utility indicates better preservation. Results for Llama and
    Gemma are reported in
    Appendix~\ref{app:additional_benchmarks}.
    }
    \label{fig:benchmark_raw}
\end{figure*}

\begin{figure*}[!ht]
    \centering
    \includegraphics[width=\textwidth]
    {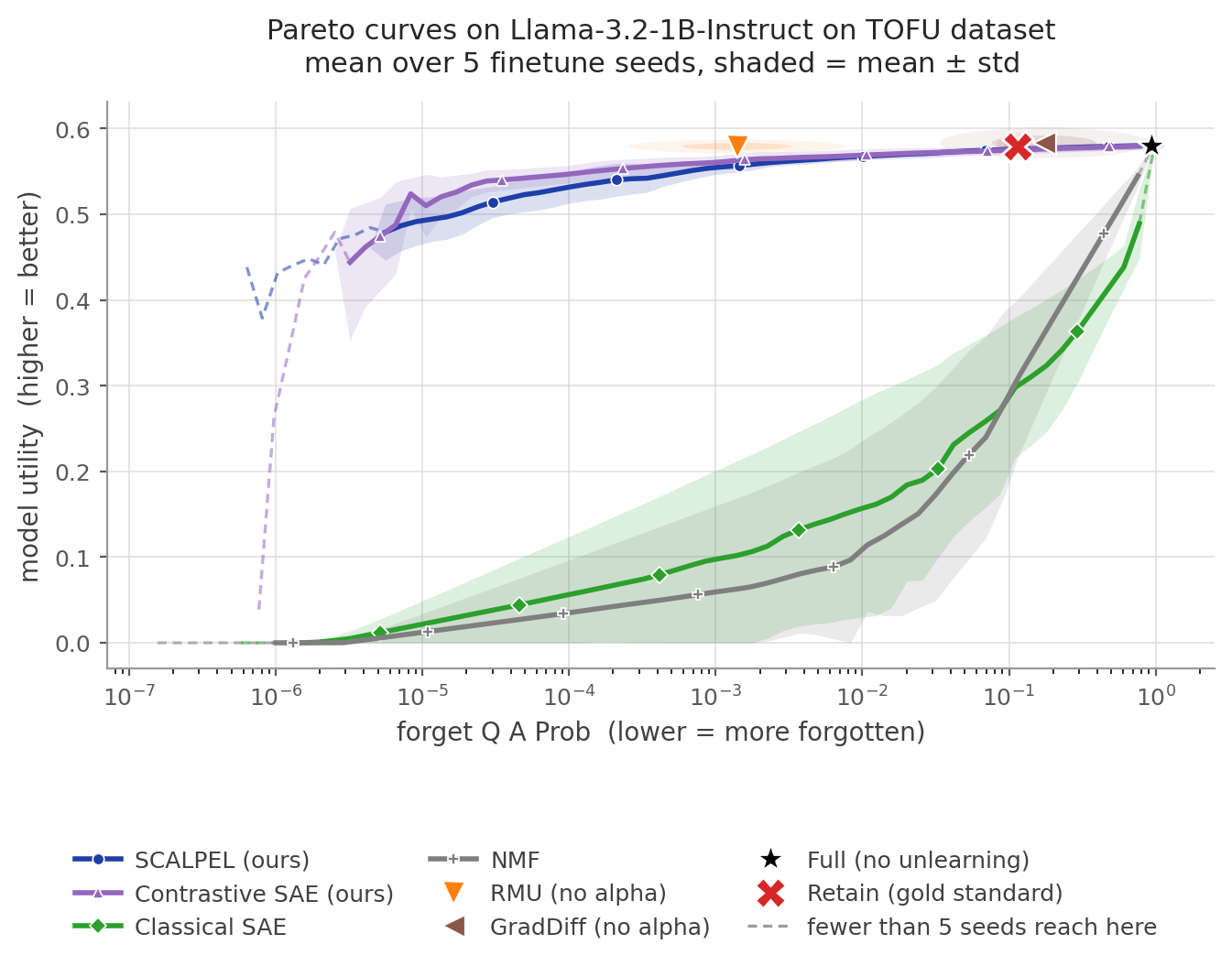}
    \caption{
    Unlearning trade-off on TOFU forget10 for
    \texttt{Llama-3.2-1B-Instruct}, after method-specific
    hyperparameter optimization and aggregation over five independent
    end-to-end seeds. Each curve varies the intervention strength:
    lower forget probability indicates stronger forgetting, while higher
    model utility indicates better preservation.
    }
    \label{fig:benchmark-llama}
\end{figure*}

\begin{figure*}[!ht]
    \centering
    \includegraphics[width=\textwidth]
    {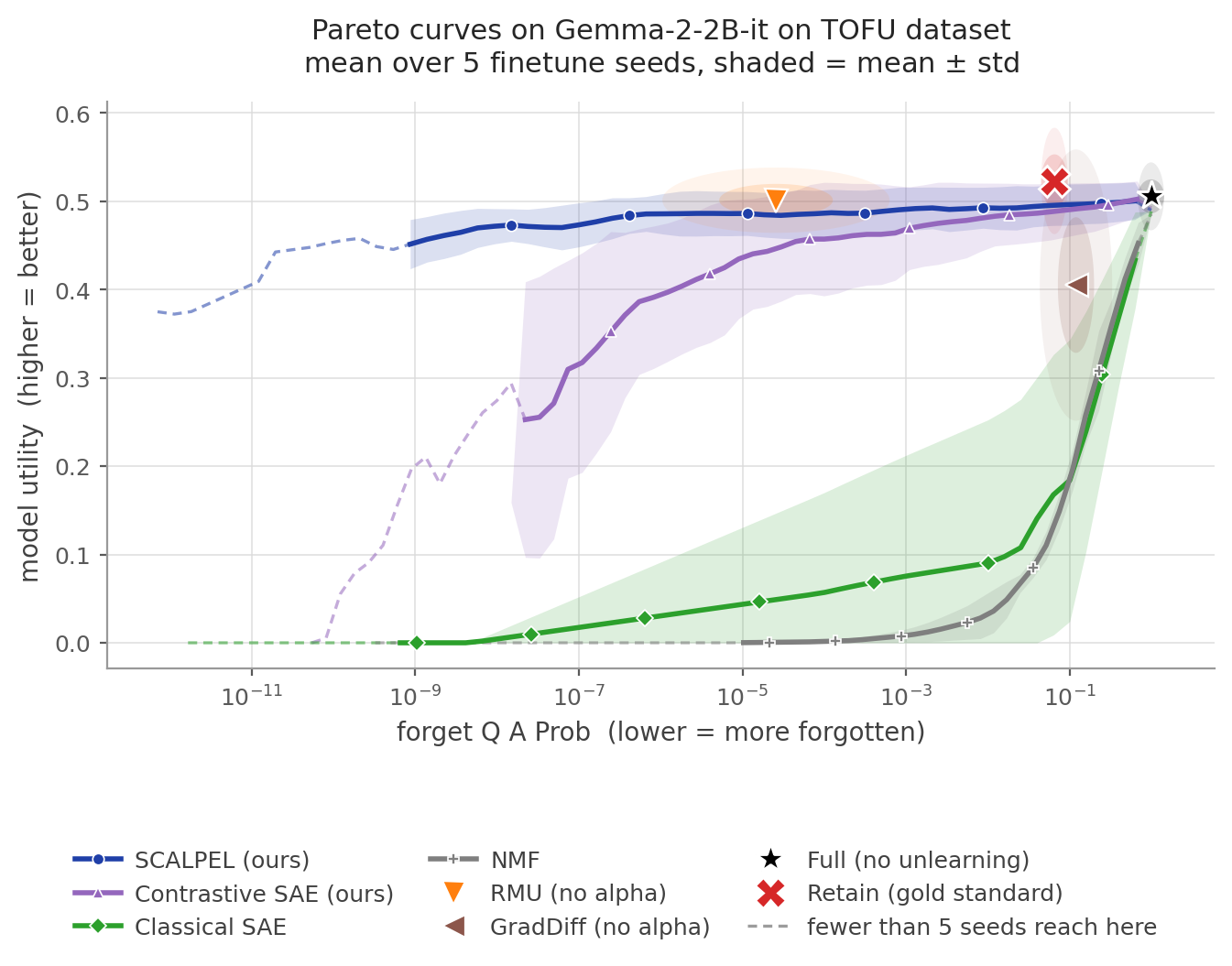}
    \caption{
    Unlearning trade-off on TOFU forget10 for
    \texttt{Gemma-2-2B-it}, after method-specific
    hyperparameter optimization and aggregation over five independent
    end-to-end seeds. Each curve varies the intervention strength:
    lower forget probability indicates stronger forgetting, while higher
    model utility indicates better preservation.
    }
    \label{fig:benchmark-gemma}
\end{figure*}

\subsection{Trade-offs Across Multiple Unlearning Metrics}
\label{app:additional_benchmarks_tradeoff_curves}

\begin{figure*}[t]
    \centering
    \includegraphics[width=\textwidth]
    {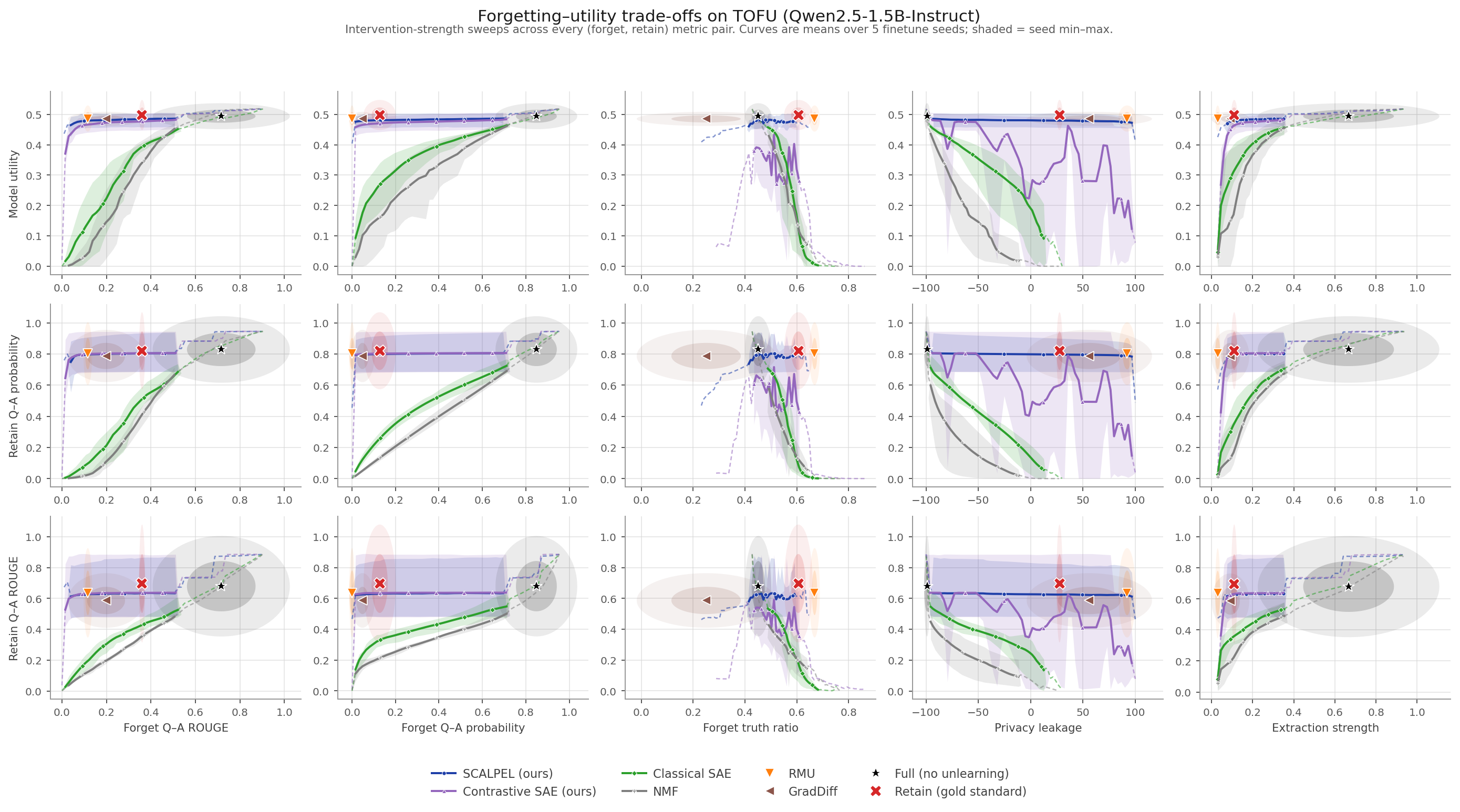}
    \caption{
    Multi-metric forgetting--utility trade-offs on TOFU forget10 for
    \texttt{Qwen2.5-1.5B-Instruct}, after method-specific hyperparameter
    optimization. Curves show the mean over five finetuning seeds, and
    shaded regions show the seed-wise minimum--maximum range.
    Columns correspond to different forget-side metrics, while rows report
    preservation through model utility, retain Q--A probability, and retain
    Q--A ROUGE, for which higher values are better.
    Lower forget Q--A ROUGE and forget Q--A probability indicate stronger
    forgetting. Forget truth ratio and extraction strength are interpreted
    relative to the retain-only reference, while privacy leakage is best
    when close to zero. The black star and red cross denote the full and
    retain-only reference models, respectively; RMU and GradDiff are shown
    at their selected operating points.
    }
    \label{fig:tradeoff-grid-qwen}
\end{figure*}

\begin{figure*}[t]
    \centering
    \includegraphics[width=\textwidth]
    {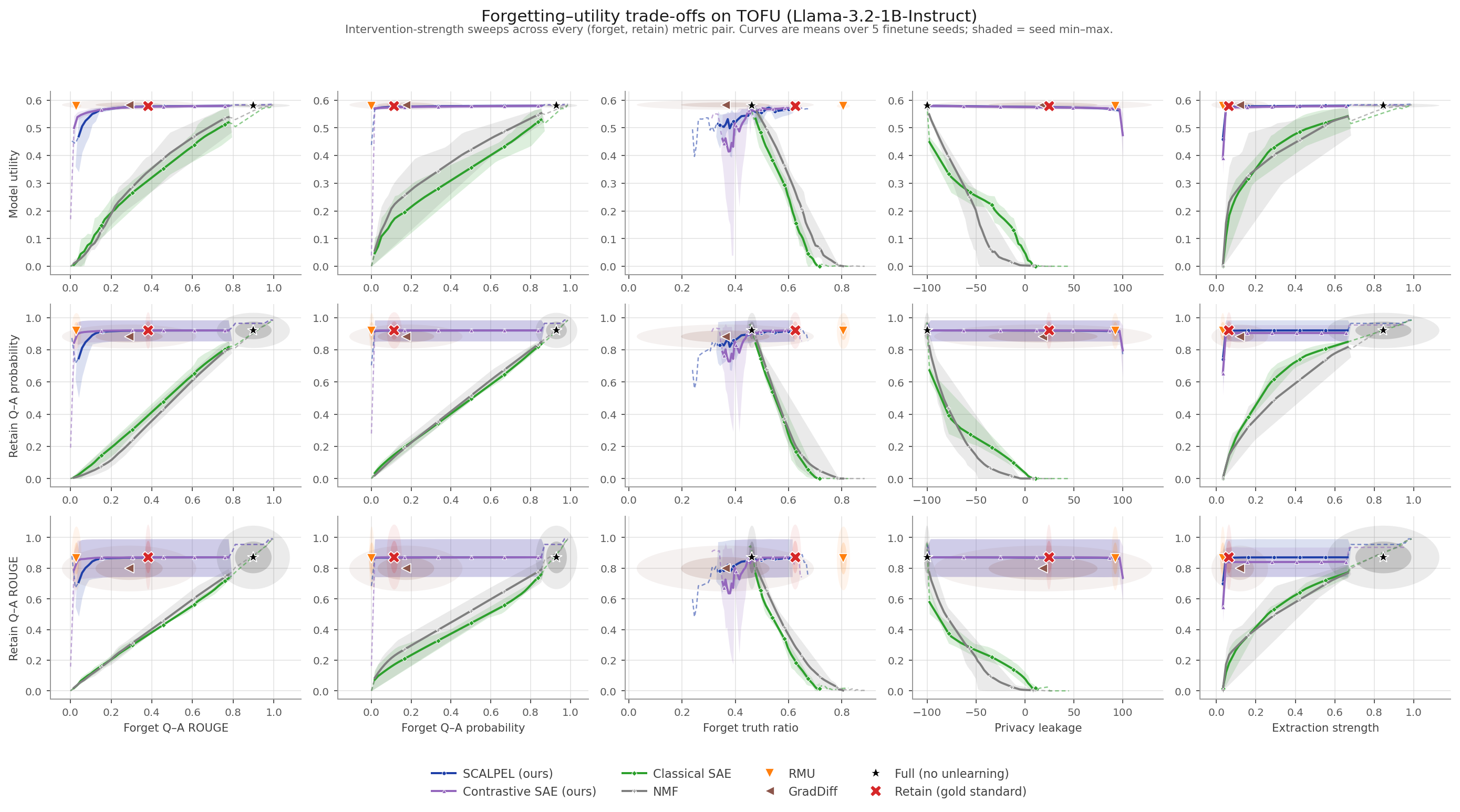}
    \caption{
    Multi-metric forgetting--utility trade-offs on TOFU forget10 for
    \texttt{Llama-3.2-1B-Instruct}, after method-specific hyperparameter
    optimization. Curves show the mean over five finetuning seeds, and
    shaded regions show the seed-wise minimum--maximum range.
    Columns correspond to different forget-side metrics, while rows report
    preservation through model utility, retain Q--A probability, and retain
    Q--A ROUGE, for which higher values are better.
    Lower forget Q--A ROUGE and forget Q--A probability indicate stronger
    forgetting. Forget truth ratio and extraction strength are interpreted
    relative to the retain-only reference, while privacy leakage is best
    when close to zero. The black star and red cross denote the full and
    retain-only reference models, respectively; RMU and GradDiff are shown
    at their selected operating points.
    }
    \label{fig:tradeoff-grid-llama}
\end{figure*}

\begin{figure*}[t]
    \centering
    \includegraphics[width=\textwidth]
    {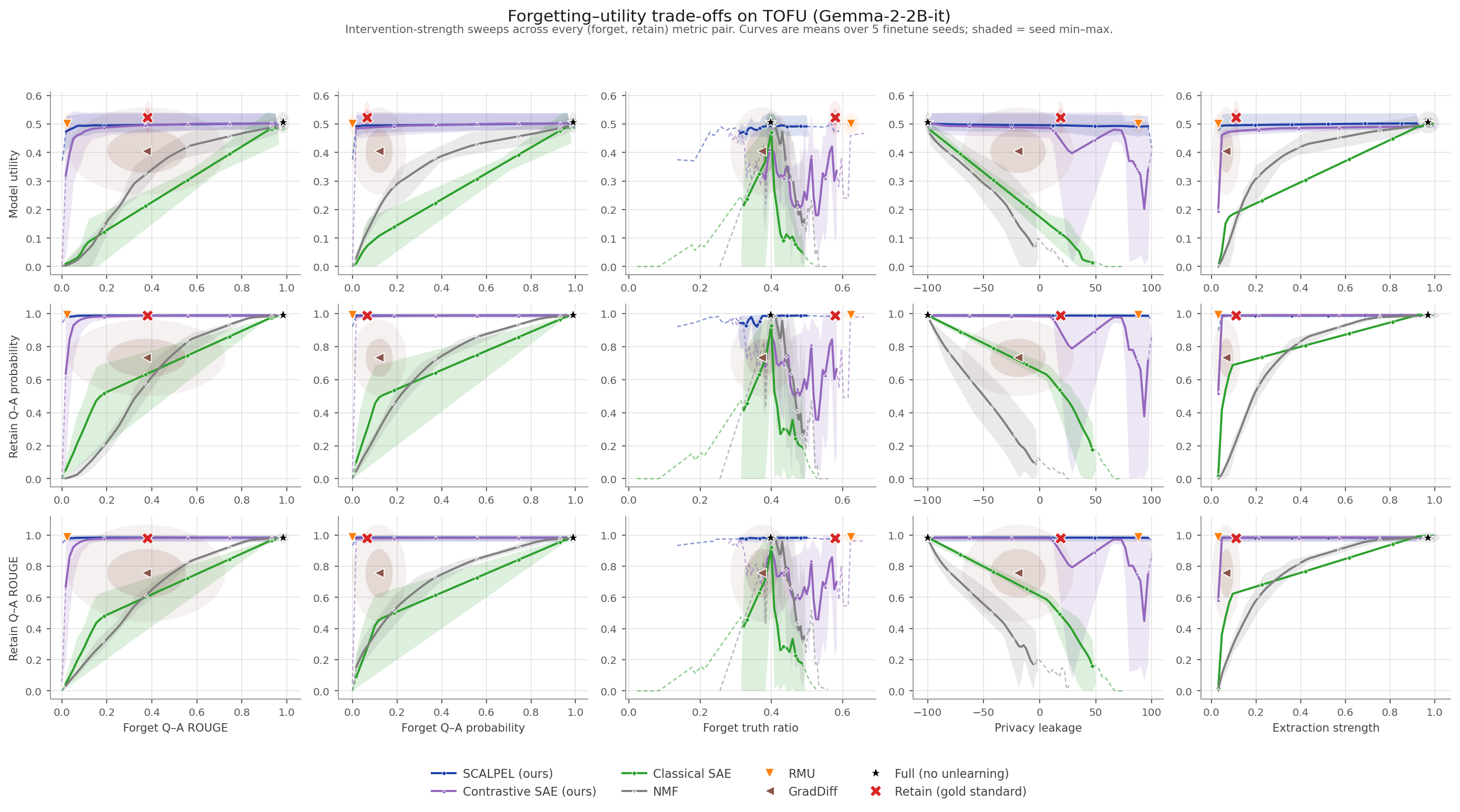}
    \caption{
    Multi-metric forgetting--utility trade-offs on TOFU forget10 for
    \texttt{Gemma-2-2B-it}, after method-specific hyperparameter
    optimization. Curves show the mean over five finetuning seeds, and
    shaded regions show the seed-wise minimum--maximum range.
    Columns correspond to different forget-side metrics, while rows report
    preservation through model utility, retain Q--A probability, and retain
    Q--A ROUGE, for which higher values are better.
    Lower forget Q--A ROUGE and forget Q--A probability indicate stronger
    forgetting. Forget truth ratio and extraction strength are interpreted
    relative to the retain-only reference, while privacy leakage is best
    when close to zero. The black star and red cross denote the full and
    retain-only reference models, respectively; RMU and GradDiff are shown
    at their selected operating points.
    }
    \label{fig:tradeoff-grid-gemma}
\end{figure*}

\section{Feature Selectivity Across Extractors}
\label{app:contrastive-selectivity}

\begin{figure*}[!ht]
    \centering
    \includegraphics[width=0.65\textwidth]
    {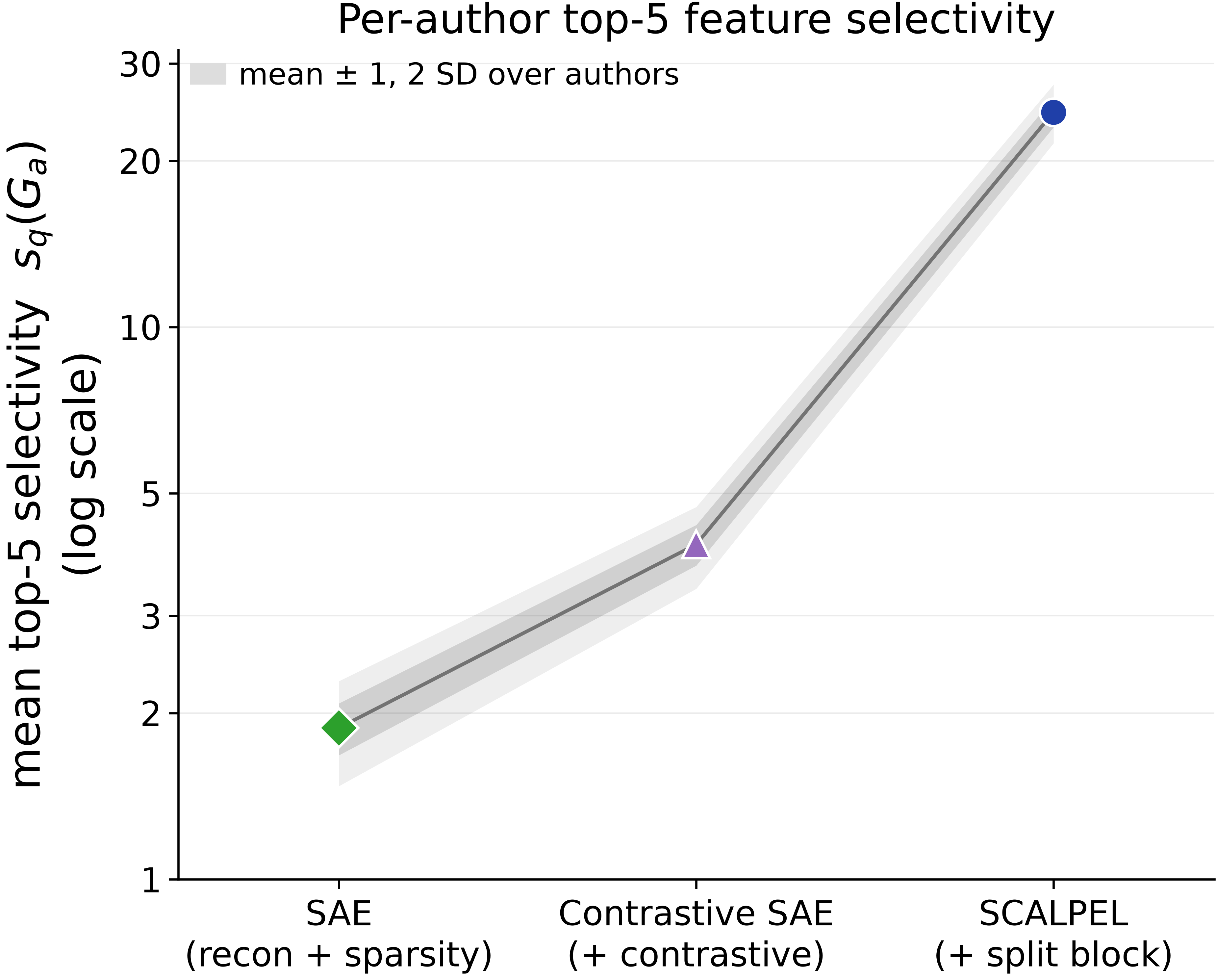}
    \caption{
    Per-author top-5 feature selectivity for the reconstruction-only SAE,
    contrastive SAE, and \textsc{SCALPEL}-Split, evaluated across the five
    Qwen benchmark seeds. Points correspond to the mean top-5
    selectivity score $s_q(G_a)$ for each author, while the shaded region
    summarizes variation across authors. Adding contrastive learning
    increases feature selectivity over the standard SAE, and the split
    architecture further amplifies this effect.
    }
    \label{fig:contrastive-selectivity}
\end{figure*}

Figure~\ref{fig:contrastive-selectivity} complements the representation-level
geometry of Figure~\ref{fig:contrastive-geometry}. The contrastive SAE produces
more selective top-ranked features than the reconstruction-only SAE, while
\textsc{SCALPEL}-Split yields the strongest selectivity. This supports the
interpretation that contrastive learning does not merely reorganize the latent
geometry globally, but translates this separation into individual features that
are more specific to the forget targets.

\section{Adaptive Intervention-Strength Search}
\label{app:alpha_search}

Most methods expose a scalar parameter that controls the strength of
unlearning; we denote it generically by $\alpha$, although its precise
meaning depends on the method (intervention strength for
representation-based methods, update scale for parameter-based ones). A
small value leaves the model close to $\theta_{\mathrm{full}}$, while
a very large value produces near-complete forgetting together with
substantial degradation. A uniform grid over $\alpha$ is inefficient:
the practically relevant transition often occupies a narrow interval,
and $\alpha$ can span several orders of magnitude. We therefore trace
each trade-off curve with an adaptive bisection in $\log\alpha$.

\paragraph{Procedure.}
\begin{enumerate}
  \item \textbf{Endpoints.} Evaluate two endpoints
  $\alpha_{\mathrm{low}}$ and $\alpha_{\mathrm{high}}$, chosen to
  bracket the transition from the nearly unchanged full model to a
  strongly degraded model. Each evaluation runs the OpenUnlearning
  pipeline and returns the pair
  $\bigl(P_f(\alpha), P_r(\alpha)\bigr)$ of forget and retain
  ground-truth probabilities. All evaluated points are cached, so an
  interrupted search resumes without recomputation.
  \item \textbf{Interval selection.} At each iteration, sort the
  evaluated points by $\alpha$ and consider every adjacent pair
  $(\alpha_i, \alpha_{i+1})$. Prioritize intervals whose
  forget-probability range overlaps the focus region
  $\mathcal{B}_f = [0.05,\, 0.70]$, where the practically relevant
  part of the trade-off lies; among intervals with the same overlap,
  select the one with the largest gap in forget probability.
  \item \textbf{Refinement.} Evaluate the geometric midpoint
  $\alpha_{\mathrm{new}} = \sqrt{\alpha_i\,\alpha_{i+1}}$, i.e.\ the
  midpoint of the interval in $\log\alpha$ space.
  \item \textbf{Stopping.} Repeat until the prescribed budget of
  evaluation points is reached or no interval contains a meaningful
  unresolved gap.
\end{enumerate}
The procedure keeps the two endpoint evaluations and concentrates the
remaining budget in the steep region of the curve, producing a dense
estimate of the Pareto front without a large uniform sweep.

\paragraph{Aggregation across seeds.}
After evaluation, the points of each seed are sorted by $\alpha$. To
aggregate the five seeds, each seed-specific curve is interpolated
onto a common forget-probability grid over the seeds' shared support;
the corresponding retain values are then averaged across seeds, and
error bars are computed from the five independent runs. This yields
comparable mean curves and uncertainty estimates even though the
adaptive search evaluates different $\alpha$ values in different
seeds.


\section{Extended Limitations and Future Work}
\label{app:limitations_full}

This appendix expands the limitations summarized in
Section~\ref{sec:conclusion}.

\paragraph{Scope of the evaluation.}
Our experiments cover TOFU and three relatively small models. Although
Qwen, Llama, and Gemma represent distinct families, further work is
needed to establish whether the same behavior holds for larger models
and less controlled datasets. We have not yet completed the evaluation
on MUSE, which would require concept groups at the document or chunk
level rather than at the author level, and on WMDP-Bio. Our theory
also predicts where the advantage should shrink: the energy-bias
argument (Appendix~\ref{App:intro_th}) implies that \textsc{SCALPEL}
helps most when target concepts carry low energy in representation
space, so the gap with existing methods may narrow on datasets where
targets already contribute strongly to the activations.

\paragraph{Concept granularity.}
The author-level definition of a concept fits the structure of TOFU
and the right-to-be-forgotten setting, where each person is naturally
a separate target. Other applications may need a different
granularity, a document, a fact, or a single sample. The anchor-based
formulation supports these settings in principle through suitable data
augmentation (each anchor's views providing the positives), but their
stability and computational cost remain to be studied. Relatedly, we
have not yet measured how the method behaves as the fraction of data
to unlearn grows; because the intervention acts locally on selected
directions, we expect robustness to this parameter, but it must be
verified.

\paragraph{Computational cost and design choices.}
The method depends on several design choices: the intervention layer,
the SAE architecture and latent dimension, the contrastive weight and
temperature, the augmentation procedure, the selection threshold, and
the strength $\alpha$. We reduce this dependence through per-method
hyperparameter optimization and multi-seed evaluation, but the full
protocol is expensive: it requires finetuning full and retain-only
references, caching token-level activations, training several sparse
dictionaries, and evaluating many intervention strengths. The quality
of the augmented data is itself a limitation, we rely on a fixed set
of paraphrasing and transformation procedures, and richer augmentation
could improve the stability of the extracted features and their
coverage of the target concept.

\paragraph{Scope of the theory.}
The theoretical results are deliberately local and rely on stated
idealizations. The energy-bias result assumes centered, uncorrelated
concept components and orthogonal per-concept reconstruction errors
(Appendix~\ref{App:intro_th}); the contrastive-selectivity analysis
assumes non-negative codes, controlled decoder scales, and that sparse
target--background separation is compatible with comparable
reconstruction and sparsity cost
(Appendix~\ref{app:sparse_contrastive}); and the perturbation bound
controls only the expected representation-level footprint of the
intervention, not downstream behavior
(Appendix~\ref{App:score_unlearing}). These assumptions make the
mechanisms analyzable but do not fully capture the geometry of
language-model representations, and none of the results guarantees
complete forgetting or behavioral preservation.

\paragraph{Score geometry.}
The score $s_q(G_a)$ uses activation magnitudes and decoder norms but
ignores the orientation of $d_q$ relative to the local forget and
background manifolds, which likely explains why features with similar
scores can produce different unlearning effects
(Section~\ref{sec:score-results}). A geometry-aware score could
compare each decoder direction with local principal directions of the
two manifolds, and a geometry-aware intervention could move the
activation along, or back toward, the retain manifold instead of
applying a straight subtraction, potentially approaching the behavior
of a model trained only on retained data. Both require repeated local
manifold estimation (e.g.\ local PCA), making selection and
intervention substantially more expensive; we view them as promising
extensions rather than free improvements.

\paragraph{Shared features across targets.}
Several forget authors select the same latent features. This sharing is useful, it
means forgetting all ten authors needs far fewer than fifty
directions, reducing memory and inference cost, but it may also
couple targets: removing a shared feature for one author could affect
another author represented by the same coordinate. Further analysis is
needed to determine whether these features capture genuinely shared
forget structure or introduce unwanted interference.

\paragraph{Inference overhead and sequential unlearning.}
\textsc{SCALPEL} adds a small inference-time overhead: the SAE encoder
and the selected decoder directions must remain available so the
intervention can be applied at the chosen layer. Sequential unlearning
is an open problem: after unlearning one group, removing a second
currently requires either retraining \textsc{SCALPEL} on the combined
forget set or stacking a second intervention, which increases
inference cost with every new request. Making the method scale to
repeated unlearning requests is an important direction.

\paragraph{Context-aware extraction.}
The encoder processes token representations individually. These
already carry contextual information from earlier transformer layers,
but the extraction mechanism does not explicitly model the full sample
context; a context-aware encoder could better isolate associations
that only become identifiable at the sentence or document level.

\paragraph{Behavioral versus parametric removal.}
Finally, two limitations concern what ``unlearning'' means here.
First, existing metrics mainly measure output behavior and cannot
reliably determine whether the target information is still stored
internally or recoverable through alternative prompts or attacks;
better evaluations of true removal versus hiding are needed. Second,
\textsc{SCALPEL} modifies activations at inference and does not erase
the parametric trace: the original information may persist at other
layers or become accessible if the intervention is disabled, bypassed,
or challenged. Our results therefore establish selective behavioral
and representation-level unlearning \emph{under the proposed
intervention}, not certified deletion from the model. A natural next
step is to use the extracted features to guide permanent parameter
updates while preserving the selectivity of the representation-level
approach; we view the present work as a first step toward bridging
mechanistic interpretability and machine unlearning.

\end{document}